\documentclass{article}

\usepackage{microtype}
\usepackage{graphicx}
\usepackage{subcaption}
\usepackage{booktabs} 

\usepackage{hyperref}

\usepackage[preprint]{icml2026}

\usepackage{amsmath}
\usepackage{amssymb}
\usepackage{mathtools}
\usepackage{amsthm}
\usepackage{enumitem}
\usepackage{float}

\usepackage[capitalize,noabbrev]{cleveref}

\theoremstyle{plain}
\newtheorem{theorem}{Theorem}[section]

\newtheorem{lemma}[theorem]{Lemma}

\theoremstyle{definition}
\newtheorem{definition}[theorem]{Definition}

\theoremstyle{remark}

\usepackage[textsize=tiny]{todonotes}

\usepackage{cuted}
\usepackage{caption} 

\icmltitlerunning{Unraveling the Hidden Dynamical Structure in Recurrent Neural Policies}

\begin{document}

\twocolumn[
  \icmltitle{Unraveling the Hidden Dynamical Structure in Recurrent Neural Policies}



  \icmlsetsymbol{equal}{*}

  \begin{icmlauthorlist}
    \icmlauthor{Jin Li}{equal,sapient,thu}
    \icmlauthor{Yue Wu}{equal,thu}
    \icmlauthor{Mengsha Huang}{thu}
    \icmlauthor{Yuhao Sun}{sapient}
    \icmlauthor{Hao He}{thu}
    \icmlauthor{Xianyuan Zhan}{thu}
  \end{icmlauthorlist}

  \icmlaffiliation{sapient}{Sapient Intelligence}
  \icmlaffiliation{thu}{Tsinghua University}

  \icmlcorrespondingauthor{Jin Li}{electrixoul@outlook.com \& jin@sapient.inc}
  \icmlcorrespondingauthor{Xianyuan Zhan}{zhanxianyuan@gmail.com}

  \icmlkeywords{Machine Learning, ICML}

  \vskip 0.3in
]



\printAffiliationsAndNotice{\icmlEqualContribution}

\newif\ifshowsuggest
\showsuggesttrue 

\newcommand{\syh}[1]{
    \ifshowsuggest
        \textcolor{red}{[SYH: #1]}
    \fi
}

\newcommand{\yue}[1]{
    \ifshowsuggest
        \textcolor{cyan}{[Yue: #1]}
    \fi
}

\begin{abstract}
Recurrent neural policies are widely used in partially observable control and meta-RL tasks. Their abilities to maintain internal memory and adapt quickly to unseen scenarios have offered them unparalleled performance when compared to non-recurrent counterparts. However, until today, the underlying mechanisms for their superior generalization and robustness performance remain poorly understood. In this study, by analyzing the hidden state domain of recurrent policies learned over a diverse set of training methods, model architectures, and tasks, we find that stable cyclic structures consistently emerge during interaction with the environment. Such cyclic structures share a remarkable similarity with \textit{limit cycles} in dynamical system analysis, if we consider the policy and the environment as a joint hybrid dynamical system. Moreover, we uncover that the geometry of such limit cycles also has a structured correspondence with the policies' behaviors. These findings offer new perspectives to explain many nice properties of recurrent policies: the emergence of limit cycles stabilizes both the policies’ internal memory and the task-relevant environmental states, while suppressing nuisance variability arising from environmental uncertainty; the geometry of limit cycles also encodes relational structures of behaviors, facilitating easier skill adaptation when facing non-stationary environments.

\end{abstract}


\section{Introduction}

Acting under partial observability and non-stationary environments often requires policies to integrate and maintain historical information in order to adapt within a task and generalize to new environments \citep{Kaelbling1998Planning,Hausknecht2015DRQN,Duan2016RL2,Wang2016LearningToRL}.
Recurrent policies are well suited for these settings: they maintain internal hidden states that integrate past observations to support decision-making \citep{Elman1990Finding,Hochreiter1997LSTM,Hausknecht2015DRQN}.
Despite their ubiquity and empirical success, the underlying properties of recurrent policies and the reasons for their superior performance 
remain poorly understood \citep{PuiuttaVeith2020,Heuillet2021}.
Their decision-making processes are obscured by the complexity of the agent-environment interaction loop~\citep{Kaelbling1998Planning,Hausknecht2015DRQN}, the evolution of hidden state patterns~\citep{banino2018vector,jensen2024recurrent}, as well as diverse recurrent model architectures and policy learning methods~\citep{Elman1990Finding,Hochreiter1997LSTM,Cho2014Learning,Schulman2017PPO,Salimans2017Evolution}, developing a unified mechanistic understanding of recurrent neural policies has proven to be difficult. 


Until today, suitable explanatory analysis tools remain lacking for current neural policies. 
Methods such as feature attributions or saliency maps are typically designed to explain short-horizon, step-level effects, primarily used for rationalizing individual actions rather than coherent strategies~\citep{PuiuttaVeith2020,Heuillet2021,Simonyan2013DeepInside,Sundararajan2017Axiomatic,Greydanus2018Visualizing,Zahavy2016Graying}.
Conversely, while dynamical analyses can reveal meaningful low-dimensional structures in recurrent networks \citep{SussilloBarak2013}, such analyses are rarely tied to policy execution in closed-loop environment interactions \citep{milani2024explainable}.

In this study, we find that surprisingly universal cyclic dynamical structures consistently emerge in the hidden state domain of fully optimized recurrent policies during their interaction with the environment. Such patterns hold across different task families,
training pipelines (gradient-based policy optimization~\citep{Schulman2017PPO} and evolution strategies~\citep{Salimans2017Evolution}), and recurrent model architectures (conventional RNN structure and modern state space models like Mamba~\citep{Gu2023Mamba}). 
Importantly, these cyclic patterns emerge in the policy's hidden state even when memory is preserved across episode boundaries, indicating that they are a property of the learned dynamics rather than an artifact of state re-initialization.

We further find that such cyclic structures can be well captured if viewing the agent execution as evolution in the state space of a hybrid dynamical system (HDS) whose joint state includes both the policy's hidden neural states and environment states, with episode resets acting as discrete events \citep{Goebel2012Hybrid,Branicky1998Hybrid,Lygeros2003DynamicalProperties}.
In this view, the observed cyclic structures correspond to attracting \textit{limit cycles} in the joint state space of HDS~\citep{Strogatz2015Nonlinear,GuckenheimerHolmes1983Nonlinear}.
These cycles serve as basic elements for orchestrating decisions over time: they stabilize the coupled evolution of policy's memory and task-relevant environment state, and are robust to external perturbations, effectively suppressing nuisance variation from unknown environments.



Beyond the existence of limit cycles, we also uncover a structured correspondence between limit cycles and policy behaviors.
Different limit cycles correspond to distinct behavioral regimes, and the relationships among behaviors are mirrored by the geometry of the corresponding limit cycles.
In particular, similarities between behaviors are strongly correlated with the geometric similarities of their limit cycles.
This indicates that recurrent policies do not merely store information, but organize it into semantically meaningful representations that preserve relational structure at the level of behaviors. This might partly explain why recurrent policies often have better task generalization capabilities in meta-RL tasks~\citep{Duan2016RL2,Wang2016LearningToRL}.
Finally, our findings resonate with recent breakthroughs in biological motor control, which reveal that neural latent dynamics are conserved across distinct individuals performing identical tasks~\citep{safaie2023preserved}. This convergence suggests a fundamental principle shared by biological and artificial intelligence: neural manifolds are geometrically constrained to scaffold consistent physical behaviors.

\section{Background}
\label{sec:background}

\subsection{Sequential Decision Making with Recurrent Policies}
We study sequential decision-making problems in which policy learning requires integrating information over time, either because observations are partial (POMDP-style)~\citep{Kaelbling1998Planning} or because the task itself is underspecified and must be inferred online (meta-RL-style)~\citep{Duan2016RL2,Wang2016LearningToRL}. A convenient unifying view is to assume there exists a latent task/context variable $c$ that remains fixed for an extended interaction, together with an environment state $s_t$ that evolves over time \citep{Rakelly2019PEARL,Zintgraf2020Varibad}. The agent observes $o_t$ generated from the environment state, selects an action $a_t$, and receives reward $r_t$, while the environment dynamics depend on both the action and the fixed context \citep{Kaelbling1998Planning,Rakelly2019PEARL}. Meta-RL is naturally captured by a \emph{trial} structure: the same context $c$ is revisited over multiple episodes, where episode termination resets the environment state (e.g., $s_t$) but keeps the context unchanged~\citep{Duan2016RL2,Wang2016LearningToRL,Rakelly2019PEARL,team2023human}. Crucially, the agent is allowed to preserve internal memory across these episode boundaries, enabling within-trial adaptation without parameter updates.

Recurrent policy is the default implementation choice for this requirement by maintaining an internal hidden state $h_t\in\mathcal{H}$ that summarizes past interaction and is updated online \citep{Elman1990Finding,Hochreiter1997LSTM}:
\begin{equation}
h_{t+1}=f_\theta(h_t,o_t), \
a_t \sim \pi_\theta(\cdot\mid o_t, h_t)\ \text{or}\  a_t=\pi_\theta(o_t, h_t).
\end{equation}
This design is widely used in both POMDP and meta-RL tasks, where carrying $h_t$ across time (and, in the trial setting, across episodes) supports online adaptation and empirically strong generalization to new task instances \citep{Hausknecht2015DRQN,Duan2016RL2,Wang2016LearningToRL}. The central question of this paper is not whether recurrent policies can store historical information, but how they organize and integrate moment-to-moment decisions into coherent long-horizon strategies, and what execution-scale structure underlies their generalization and adaptation.



\subsection{Preliminaries for Dynamical System Analysis}\label{sec:pre_dyn}
\textbf{Hybrid dynamical system (HDS).}\quad
To conduct our investigation, we borrow basic concepts from dynamical system analysis~\citep{Strogatz2015Nonlinear,Goebel2012Hybrid}.
The key step is to analyze the \emph{coupled interaction} between agent and environment, rather than the policy in isolation \citep{Kaelbling1998Planning}.
We treat all variables that evolve during execution on equal footing as a joint state $x_t=(s_t, h_t)$, where $s_t$ denotes the task/environment states and $h_t$ the policy's hidden neural state.
The rollout is then modeled as a discrete-time \emph{hybrid dynamical system} governed by the interplay between the agent's internal dynamics $F_\theta$ (memory update) and the environment's transition dynamics $G$.
Crucially, the system is hybrid because the environment dynamics $G$ define a \emph{Jump Set} $\mathcal{D} \subset \mathcal{S}$ (e.g., goal locations or termination states) that triggers discrete events.
The joint evolution $\Phi_\theta$ switches behavior based on whether the physical state enters this set:
\begin{equation}
x_{t+1}=\Phi_\theta(x_t)=
\begin{cases}
\mathcal{T}_{\text{flow}}(x_t), & s_t \notin \mathcal{D},\\
\mathcal{T}_{\text{jump}}(x_t), & s_t \in \mathcal{D},
\end{cases}
\label{eq:hds}
\end{equation}
where $\mathcal{T}_{\text{flow}}$ represents the nominal coupled evolution (driven by $F_\theta$ and standard transitions in $G$), and $\mathcal{T}_{\text{jump}}$ represents the episodic reset dynamics.
In the context of Meta-RL, this jump map $\mathcal{T}_{\text{jump}}$ is distinctively structured: it re-initializes the physical state $s_t$ while preserving the policy's hidden memory $h_t$~\citep{Duan2016RL2,Wang2016LearningToRL}, thereby allowing the agent dynamics $F_\theta$ to accumulate history and enable adaptation across episodes.


\begin{figure*}[t]
  \centering
    \begin{subfigure}[b]{0.48\textwidth}
        \centering
        \includegraphics[width=0.6\linewidth]{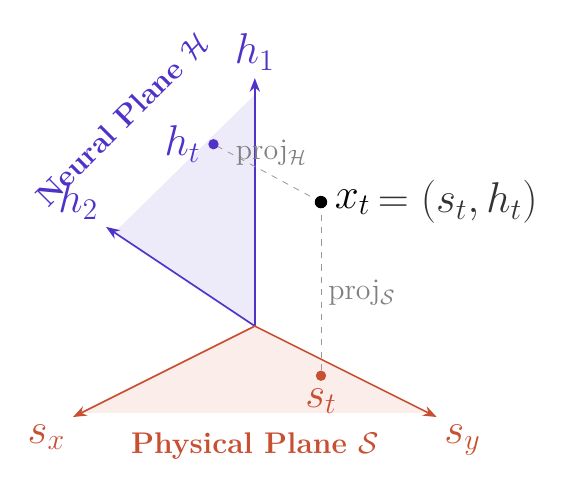}
        \caption{Unified State Space}
        \label{fig:unified_space}
    \end{subfigure}
    \hfill 
    \begin{subfigure}[b]{0.48\textwidth}
        \centering
        \includegraphics[width=0.7\linewidth]{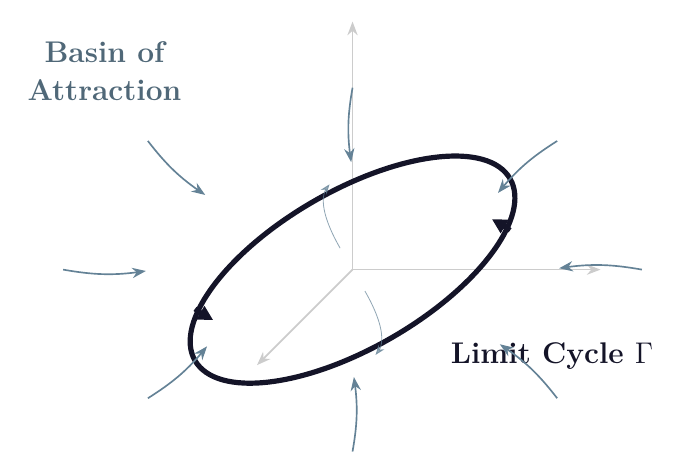}
        \caption{Example of a Limit Cycle}
        \label{fig:attracting_dynamics}
    \end{subfigure}
    \vspace{-4pt}
      \caption{\textbf{Illustration of the Hybrid Dynamical System (HDS) framework.} 
\textbf{(a) Unified State Space.} The execution state $x_t = (s_t, h_t)$ couples physical coordinates (orange axes, $\mathcal{S}$) and neural hidden states (purple axes, $\mathcal{H}$) into a single high-dimensional product space. 
\textbf{(b) Example of a Limit Cycle.} The system converges to a stable periodic orbit $\Gamma$. The vector field visualizes the basin of attraction, showing how trajectories from diverse initial conditions rapidly contract onto this cycle.}
      \label{fig:hds}
  \vspace{-8pt}
\end{figure*}

This modeling choice is useful because it turns long-horizon policy execution into a state-space object: once the interaction is expressed as an HDS on $x_t$, we can leverage established dynamical system analysis tools to characterize how trajectories are organized, and in particular to search for stable structures that can serve as building blocks for long-horizon strategies \citep{Strogatz2015Nonlinear}.

\textbf{Stable structures and limit cycles.}\quad
Dynamical systems often use \emph{stable structures} in state space to describe long-run behavior,
such as fixed points, attracting manifolds (continuous attractors), and limit cycles \citep{Strogatz2015Nonlinear,GuckenheimerHolmes1983Nonlinear}.
An attractor is a set toward which trajectories initialized nearby converge under repeated application of $\Phi_\theta$ \citep{Strogatz2015Nonlinear}.
In our setting, attracting limit cycles are particularly relevant: a limit cycle is an isolated periodic orbit that is stable to small perturbations transverse to the orbit, meaning executions return to the same cycle after disturbances \citep{Strogatz2015Nonlinear,GuckenheimerHolmes1983Nonlinear}.
Topologically, an attracting limit cycle is equivalent to a circle, which induces a natural phase coordinate along the orbit; this phase can persist over long time scales and provides a robust substrate for organizing long-horizon behavior \citep{Strogatz2015Nonlinear}.

To quantify system stability, a commonly used tool in dynamical system analysis is the $K$-step finite-time Lyapunov indicator (FTLI)~\citep{Wolf1985Lyapunov,Shadden2005FTLE}. It can be computed by simulating \textit{paired executions}: a nominal trajectory $\{x_{t}\}$ and a ``shadow'' trajectory $\{\tilde{x}_{t}\}$ initialized with a tiny controlled perturbation at time $t_{0}$, and tracking the evolution of their separation $\delta_{t}:=\|x_{t}-\tilde{x}_{t}\|$ using a normalized norm. Specifically, we define $\lambda_{FTLI}$ as the \textbf{average exponential rate of divergence} between these trajectories over the horizon $K$:
\begin{equation}
\label{eq:ftli}
\lambda_{FTLI}(t, K) = \frac{1}{K-1} \sum_{k=0}^{K-1} \log \frac{\delta_{t+k+1}}{\delta_{t+k}}
\end{equation}
Crucially, the value of $\lambda_{FTLI}$ reveals the local geometric structure of the system~\citep{Wolf1985Lyapunov, Strogatz2015Nonlinear}:
\begin{itemize}[leftmargin=*, noitemsep, topsep=0pt]
    \item \textbf{Convergence ($\lambda < 0$):} The shadow trajectory is pulled back toward the nominal path (contraction), which is the defining signature of limit cycles.
    \item \textbf{Divergence ($\lambda > 0$):} The trajectories drift apart exponentially, indicating sensitivity to initial conditions (chaos).
    \item \textbf{Neutrality ($\lambda \approx 0$):} Deviations persist without significant growth or decay, indicating a marginally stable or critical state between ordered and chaotic dynamics.
\end{itemize}

\begin{figure*}[t!]
  \centering
  \includegraphics[width=1.0\textwidth]{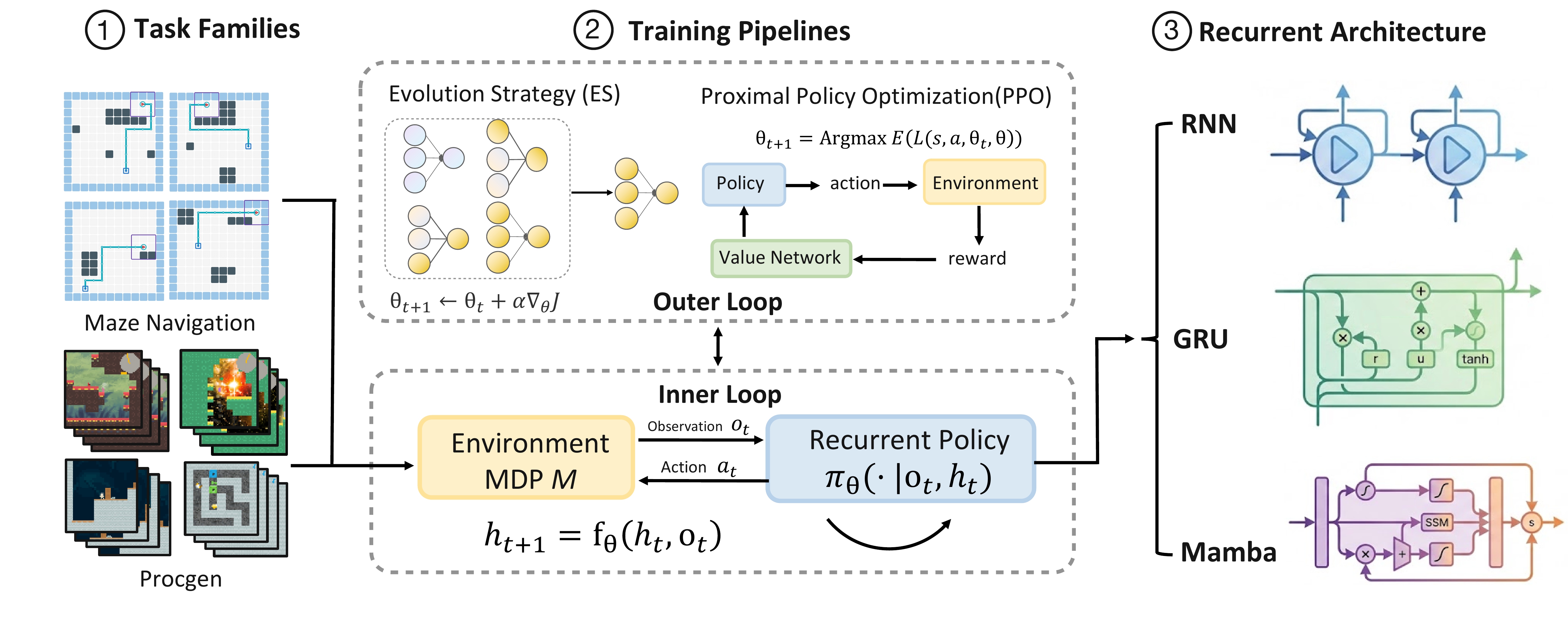}
  \vspace{-10pt}
  \caption{\textbf{Overview of the experimental framework.} We systematically evaluate recurrent policy dynamics across three axes of variation: \textbf{(1) Task Families}, ranging from partially observed grid maze navigation to high-dimensional Procgen games; \textbf{(2) Training Pipelines}, comparing gradient-based policy optimization method (PPO) against gradient-free evolution strategies; and \textbf{(3) Recurrent Architectures}, spanning classic RNNs, gated recurrent units (GRUs), and modern state-space model architecture Mamba.}
  \label{fig:setup}
\end{figure*}

\section{Experimental Design}
To unravel the underlying working mechanisms of recurrent neural policies, we carefully design a systematic experimental framework to train recurrent policies over a diverse span of tasks, training schemes, and recurrent architectures.
We are particularly interested in analyzing fully optimized recurrent policies that exhibit generalization capability. To achieve this, we adopt an RL$^2$-style training protocol \citep{Duan2016RL2}: for each trial, we sample a task instance (context) and run multiple episodes under that fixed task, where episode termination resets the environment state, but the policy retains its internal memory across episode boundaries.
We train agents to maximize return aggregated over the full trial, so that adaptation must be implemented through recurrent state dynamics rather than parameter updates.
The full training details are provided in the Appendix~\ref{app:training}.

Specifically, we vary three axes of the learning system:
\begin{itemize}[leftmargin=*, itemsep=1pt, topsep=0pt]
\item \textbf{Task family.} We evaluate both a partially observed grid maze navigation suite (procedurally generated mazes with varying origins and destinations) and several procedurally generated game-like environments from Procgen (e.g., CoinRun, Heist, Jumper) \citep{Cobbe2020Procgen}, which differ in observation structure, dynamics, and long-horizon goals. Solving these tasks requires the policy to effectively integrate information over extended temporal horizons and exhibit strong generalization capability.
\item \textbf{Training pipeline.} We train policies using both gradient-free evolution strategies (ES)~\citep{Salimans2017Evolution} and the gradient-based method PPO~\citep{Schulman2017PPO}, which allows testing if universal patterns emerge
under drastically different policy optimization procedures.
\item \textbf{Recurrent architecture.} We implement recurrent policies with vanilla RNNs, GRUs, and modern state-space architecture Mamba~\citep{Elman1990Finding,Cho2014Learning,Gu2023Mamba}, in order to identify fundamental mechanisms that are architecture-agnostic.
\end{itemize}

This setup allows us to carefully isolate impacts from task characteristics, policy optimization schemes, and recurrent architectures, thereby helping to locate the fundamental and universal mechanisms that contribute to the superior generalization and robustness performance of recurrent policies.


\section{Emergence of Limit Cycles}
\subsection{Recurrent Policies Stabilize into Limit Cycles}
With the fully optimized recurrent policies, we test these policies' behavior when solving new task instances in the corresponding environment. Again, we follow the episodic meta-RL setting, that let the agent continously solving the corresponding task through a series of trials. Although the physical environment state resets at episode boundaries, the agent's internal memory (hidden neural states) persists, enabling adaptation across episodes. Typically, with a few trials, the optimized recurrent agents start to be able to solve the unseen tasks and gradually stabilize into convergent behaviors. We collect the hidden neural state sequences of the recurrent policies in such stabilized stages, and visualize them by performing dimension reduction using principal component analysis (PCA), as illustrated in Figure~\ref{fig:limit-cycle}(a).


Interestingly, we observe that across diverse task settings (ranging from maze navigation to high-dimensional Procgen games), training pipeline (ES and PPO), and different recurrent architectures (RNN, GRU, Mamba), adapted recurrent agents always settle into stable periodic loopy structures. Some example loopy structures are presented in Figure~\ref{fig:limit-cycle}(a), more can be found in Appendix~\ref{app:lc}.
Once the policy has identified the right task context through the hidden state update, the coupled agent--environment system executes a repetitive behavioral pattern. Importantly, this emergent periodicity displays two defining dynamical properties:
\begin{itemize}[leftmargin=*, itemsep=1pt, topsep=0pt]
    \item \textbf{Emergent neural loops:} Although the hidden state $h_t$ is not externally reset at episode boundaries, the neural trajectory spontaneously converges to a closed loop in low-dimensional projections. As shown in Figure~\ref{fig:limit-cycle}(a), the neural state traces a consistent periodic orbit that persists across later episodes, maintaining a cyclic pattern through the policy's internal dynamics. 
    \item \textbf{Stability under perturbation:} The coupled system also exhibits stability under external disturbances. As depicted in Figure~\ref{fig:perturbation}, when either the physical environment state or the internal neural state is perturbed, the resulting neural loop only deviates transiently from the stable structure and then quickly recovers.
\end{itemize}

\begin{figure*}[t]
  \centering
  \includegraphics[width=0.95\textwidth]{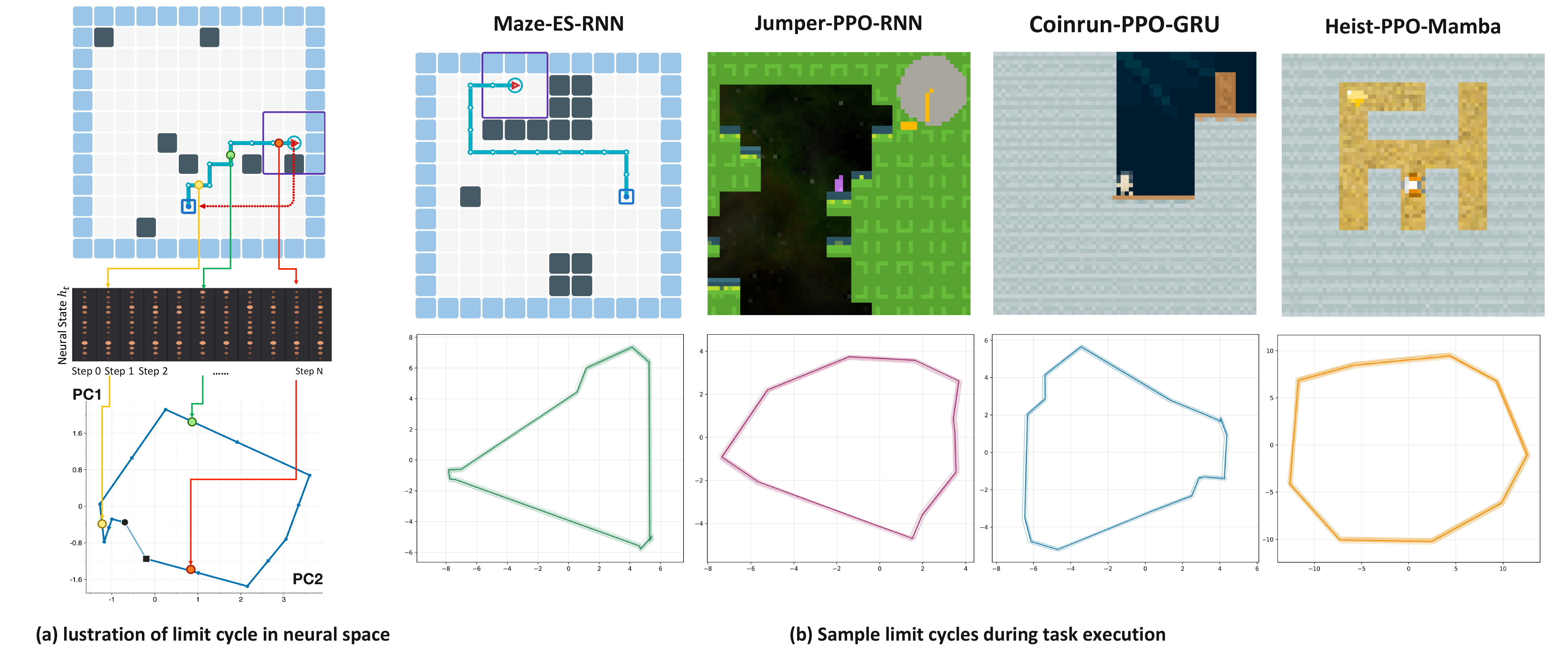}
  \caption{\textbf{Task-adapted policies stabilize into limit cycles.}
  \textbf{(a)} In the stable regime, the physical trajectory forms a topological loop enabled by episodic resets, while neural memory traces a low-dimensional closed loop (PCA projection).
  \textbf{(b)} Across task families and instances, the recurrent hidden-state trajectory converges to a stable closed orbit with task-dependent shape.}
  \label{fig:limit-cycle}
\end{figure*}

\begin{figure}[t]
  \centering
  \includegraphics[width=\columnwidth]{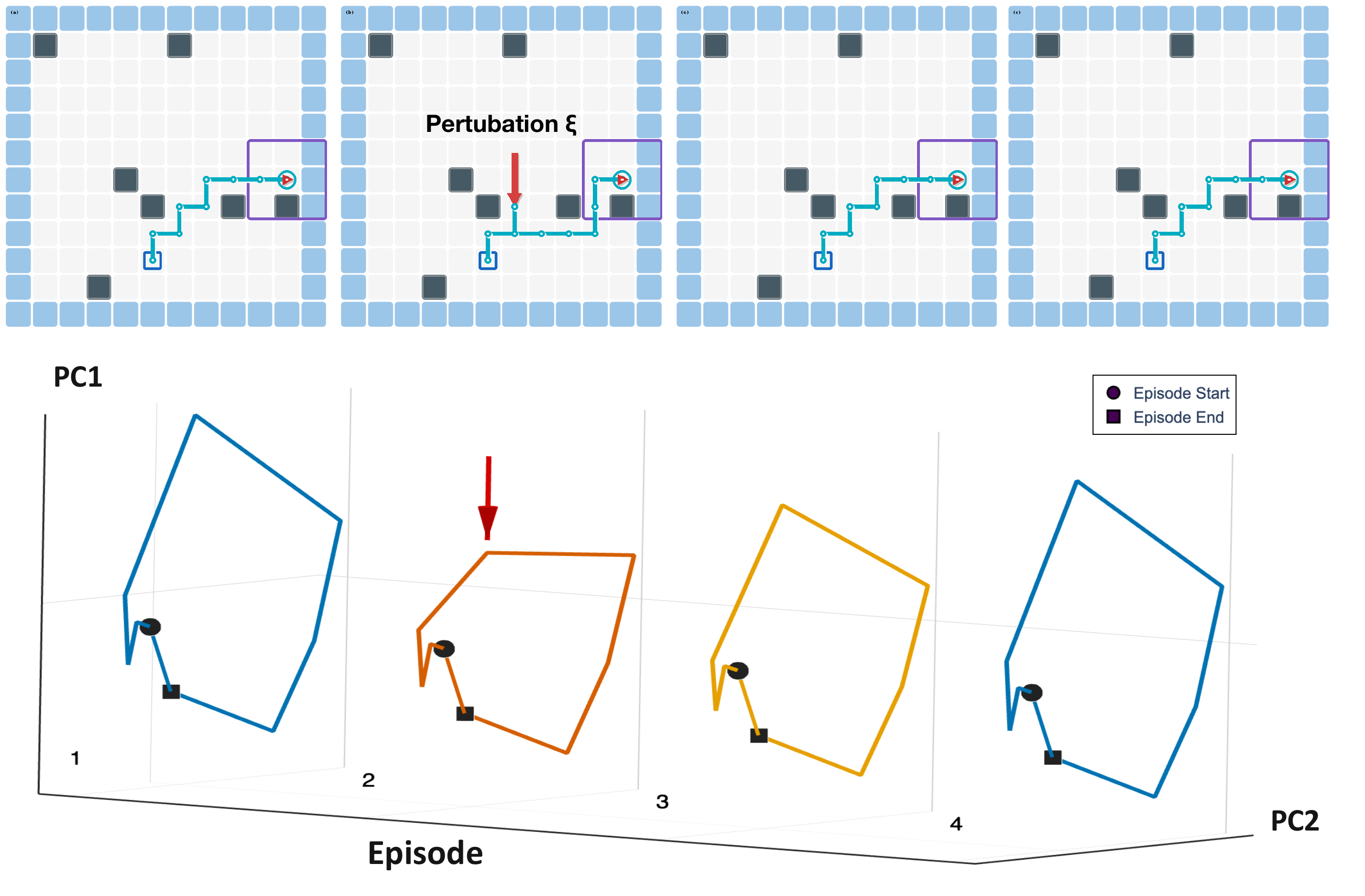}
  \caption{\textbf{Robustness and recovery under external perturbations.}
  The cycle acts as an attractor: after an external perturbation, execution transiently deviates but rapidly re-synchronizes with the nominal limit cycle. Additional experimental results on perturbation are provided in Appendix~\ref{app:perturb}.}
  \label{fig:perturbation}
  \vspace{-10pt}
\end{figure}

These properties---periodicity and convergence following perturbation--- share remarkable similarities in both topology and dynamics properties with the concept of \emph{limit cycles} in the hybrid dynamical systems (HDS) defined on the joint state $x_t=(s_t, h_t)$ of the environment state $s_t$ and hidden neural state $h_t$, as introduced in Section~\ref{sec:pre_dyn}. 


To concretely verify whether such loopy structures are indeed limit cycles, we evaluate the local stability of these structures by calculating the Time Lyapunov Indicator (FTLI) using the method described in Section~\ref{sec:pre_dyn}.
The FTLI distribution of fully optimized recurrent policies computed along the converged regime consistently concentrates below zero across diverse tasks, training methods, and recurrent architectures, as illustrated in Figure~\ref{fig:lyapunov}.
A negatively centered FTLI distribution indicates volume contraction and asymptotic stability. In contrast, randomly initialized networks evaluated under the same protocol exhibit positive exponents, indicating divergence.

The recurrence of such cyclic topological structures implies that they are stable outcomes after policy learning. The attracting nature of the limit cycle reduces sensitivity to variability arising from environmental uncertainty, maintaining hidden neural state coherence in test scenarios, thereby offering recurrent policies better generalization and robustness performance as compared to non-recurrent policies.

\begin{figure*}
  \centering
  \begin{minipage}{\textwidth}
    \centering
    \begin{subfigure}[b]{0.32\textwidth}
        \centering
        \includegraphics[width=\linewidth]{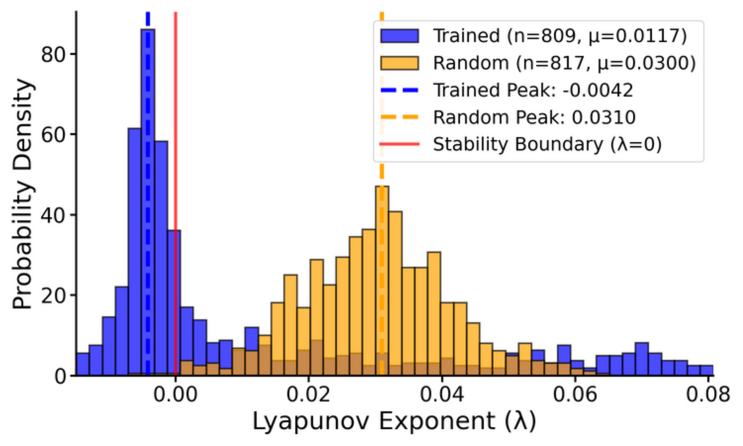}
        \caption{Maze-ES-GRU}
        \label{fig:gru_maze}
    \end{subfigure}
    \hfill 
    \begin{subfigure}[b]{0.32\textwidth}
        \centering
        \includegraphics[width=\linewidth]{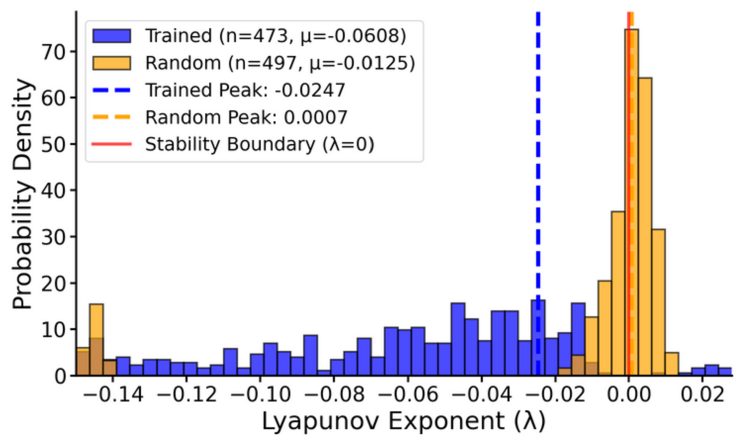}
        \caption{Jumper-PPO-GRU}
        \label{fig:gru_jumper}
    \end{subfigure}
    \hfill 
    \begin{subfigure}[b]{0.32\textwidth}
        \centering
        \includegraphics[width=\linewidth]{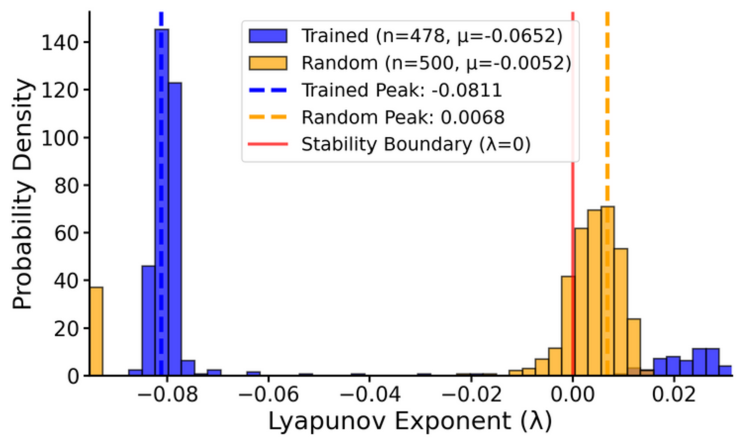}
        \caption{Jumper-PPO-Mamba}
        \label{fig:mamba_jumper}
    \end{subfigure}
    
    \caption{\textbf{Universality of contractive dynamics across architectures, tasks, and algorithms.} 
    FTLI distributions (computed with a horizon $K=1000$) for paired rollouts under small perturbations across three settings: \textbf{(a)} GRU/ES on Mazes; \textbf{(b)} GRU/PPO on Jumper; \textbf{(c)} Mamba/PPO on Jumper. Optimized policies (blue) consistently show $\lambda_{FTLI} < 0$ (contraction), while random networks (yellow) show $\lambda_{FTLI} > 0$ (chaos).}
    \label{fig:lyapunov}
  \end{minipage}
\end{figure*}

\begin{figure*}[t]
  \centering
  \begin{minipage}{\textwidth}
    \centering
    \includegraphics[width=0.95\textwidth]{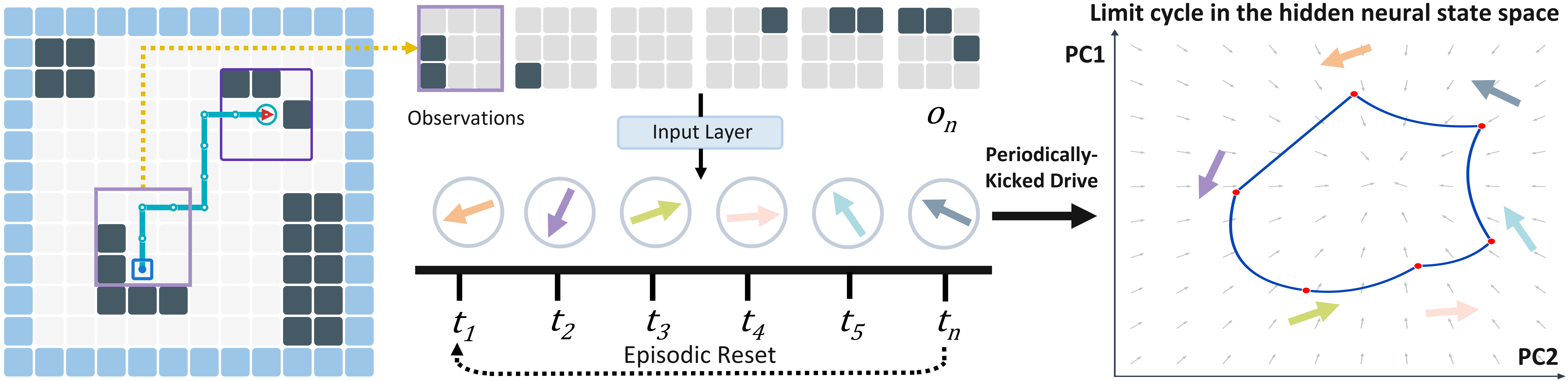}
    \caption{\textbf{Periodically-Kicked Drive (PKD) mechanism.} 
    \textbf{(Left)} Episodic resets creates a quasi-periodic \textit{Observation Sequence}. 
    \textbf{(Middle)} These observations act as a \textit{Periodically Kicked Drive} on the recurrent dynamics. 
    \textbf{(Right)} Due to the network's dissipative nature, this rhythmic forcing entrains the hidden state into a stable \textit{limit cycle}.}
    \label{fig:pkd}
  \end{minipage}
\end{figure*}

\subsection{Why Limit Cycles Arise: Episodic Task Structure Supports Periodically-Kicked Drive}
\label{sec:pkd}

Although the periodicity of the environment state can be enforced by episodic resets, the neural state preserves memory across episodes, yet still converges to stable closed loops. This 
suggests that such neural periodicity is more likely a result of the interplay between the policy's intrinsic dynamics and environment inputs.


To explain this phenomenon, we adopt the \emph{Periodically-Kicked Drive} (PKD) perspective from dynamical systems. PKD is a standard setting for analyzing how repeated periodic perturbations reshape the long-term behavior of nonlinear systems, ranging from stable periodic motion to more complex dynamics~\citep{LinYoung2010KickedOscillators, WangYoung2003StrangeAttractors}. In our setting, it offers a natural way to reason about how episodic forcing interacts with dissipative recurrent dynamics to support stable structures such as limit cycles~\citep{LohmillerSlotine1998Contraction, RussoDiBernardoSontag2010Entrainment}.


As visualized in Figure~\ref{fig:pkd}(left), the episodic reset structure guides the agent to encounter the same task conditions repeatedly. This encourages the agent to generate consistent action sequences, guiding the stabilization of the agent's behavior, eventually producing a quasi-periodic stream of observations. This observation stream functions as a periodic external ``kicked drive'' acting on the recurrent network (Figure~\ref{fig:pkd}(middle)).

Dynamical systems analysis suggests that applying a PKD to a dissipative system entrains the system's state into a stable limit cycle. We refer the interested readers to Theorem~\ref{thm:main_proof} in Appendix~\ref{app:proof_limit_cycle} for more in-depth discussion and the theoretical connection between PKD and limit cycle formation in HDS.
The recurrent policy exhibits the required dissipative properties: \textbf{local linearization analysis confirms that the operational orbit lies within a strictly contractive region of the phase space} (Appendix~\ref{app:empirical_contractivity}). This local evidence, combined with the negatively centered FTLI distribution (Figure~\ref{fig:lyapunov}), indicates that the periodic forcing provided by the observation sequence effectively entrains the hidden state into a stable limit cycle (Figure~\ref{fig:pkd} (right)).

Crucially, the existence of these limit cycles is not a trivial result of the environment reset mechanism. While resets provide the opportunity for periodicity, they do not guarantee the formation of a \textit{self-sustaining} cycle. For a limit cycle to persist in the HDS, the neural orbit must generate the exact sequence of actions that reproduces its own driving observations. 
We utilized this constraint to empirically verify the limit cycle hypothesis. By employing the PKD mechanism as a probe on large-scale randomly initialized hidden neural states ($N=100,000$), we successfully isolated stable orbits. As detailed in Appendix~\ref{app:acf} and visualized in Figure~\ref{fig:acf_clusters}, only a specific subset of these entrained orbits satisfy the causal consistency check. The successful extraction of these high-precision closed loops—which exist only for trained networks and not for random baselines—provides constructive evidence that the observed structures are functional limit cycles inherent to the learned dynamics, rather than artifacts of environmental resets.
\section{Structural Correspondence Between Behavior and Neural Dynamics}

Simply identifying the limit cycle structures hidden in recurrent policies only told half the story. We are more interested in understanding the meaning of these limit cycles: what do they represent, and how do they shape the physical behaviors of recurrent agents in the environment?

\begin{figure*}[t]
  \centering
    \includegraphics[width=0.8\textwidth]{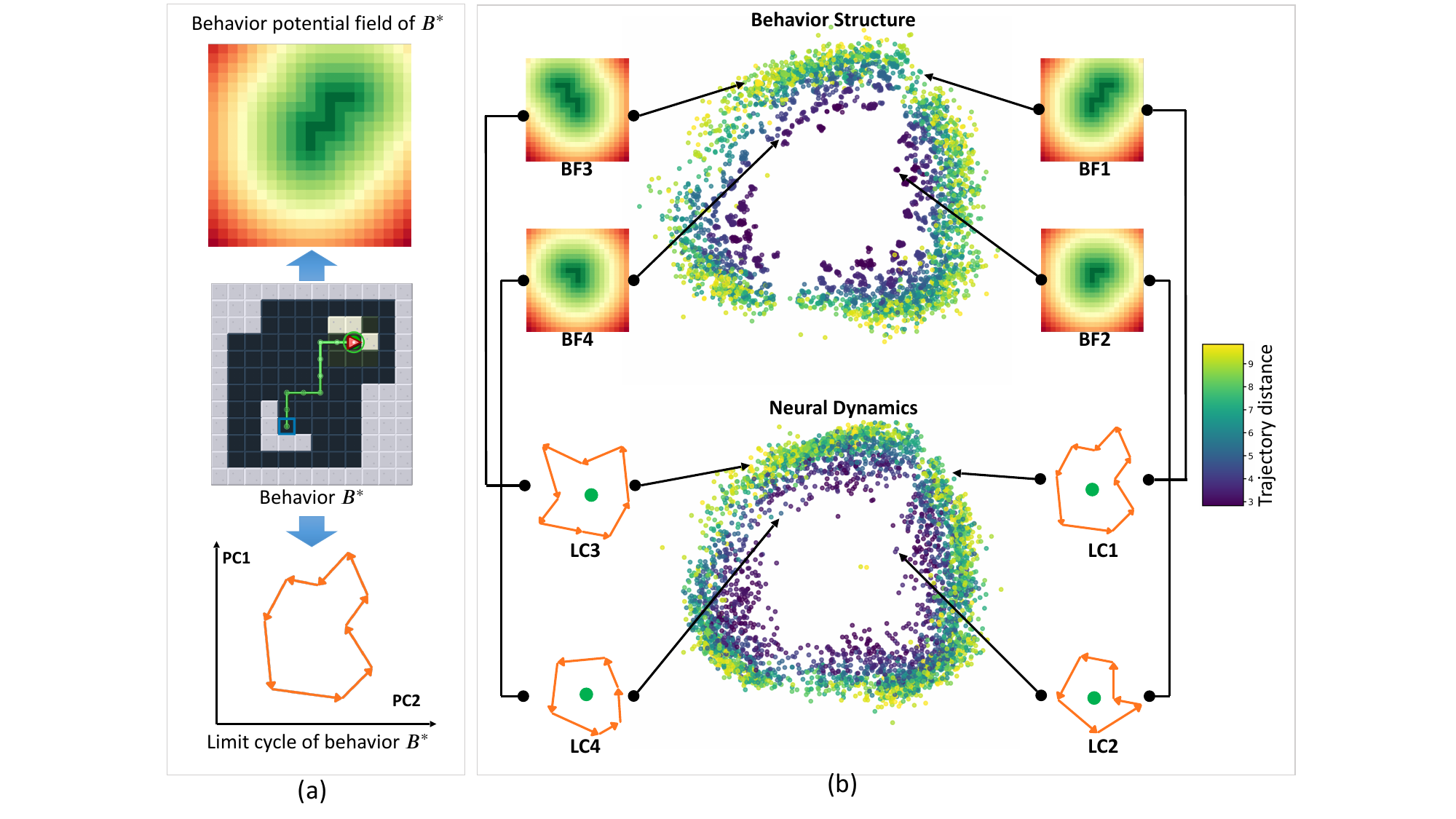}
    \caption{\textbf{Structural isomorphism between behavior and neural dynamics.} 
\textbf{(a)} A behavioral trajectory $B^*$ (represented as a Behavioral Potential Field) maps to a unique Neural Limit Cycle. 
\textbf{(b)} CCA alignment of the behavioral manifold (top) and neural dynamics manifold (bottom), visualized with the first two canonical dimensions (CM). Colored arrows highlight the correspondence between specific behavioral fields (BPF1--4) and neural attractors (LC1--4), showing that topological relationships are preserved across the mapping.}
    \label{fig:cca-overview}
\end{figure*}

To investigate this, we generate a large-scale dataset of paired samples of agent behaviors and their corresponding limit cycles. Specifically, we use the PKD mechanism described in the previous Section, by replaying observation sequences from diverse behaviors, we drive the recurrent agent into stabilized periodic regimes. To isolate causal structures, we apply an \emph{Action-Consistency Filter} (see Appendix~\ref{app:acf}) that retains only those neural cycles capable of reproducing the action sequence that generated their driving observations. This rigorous selection process yields a comprehensive library of validated limit cycles. Moreover, to model the highly diverse agent behaviors into comparable representations, we represent the agent behaviors using the \emph{Behavioral Potential Field} (BPF). Taking the maze navigation task as an example, as shown in Figure~\ref{fig:cca-overview}(a), we embed the variable-length agent trajectories into fixed-dimensional scalar fields where intensity decays with distance from the trajectory. The Euclidean distance between two BPFs approximates the transport cost between trajectories, providing a metric for behavioral similarity. For the limit cycle neural dynamics representation, we compute the centroid of each limit cycle in the neural hidden state space.

\begin{figure*}[t]
  \centering
    \includegraphics[width=\textwidth]{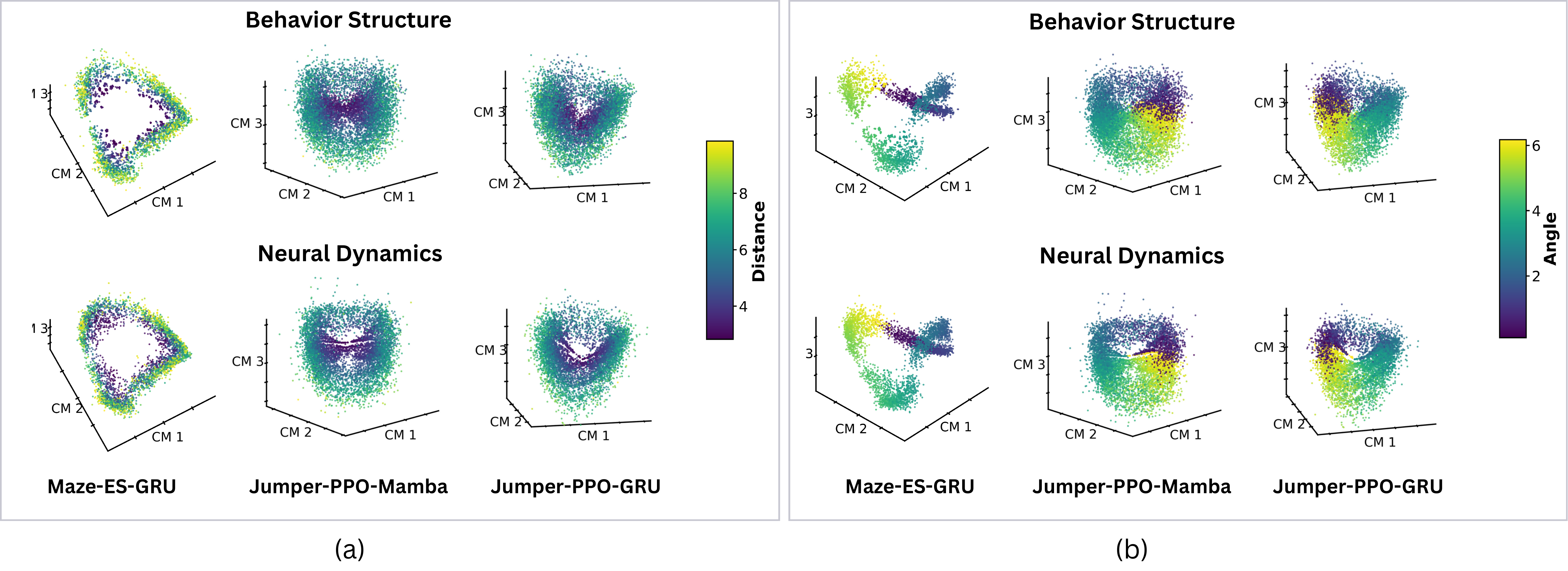}
    \caption{\textbf{Correlated geometric correspondence between behavioral manifolds and neural limit cycles.} 
    CCA projections across diverse experimental configurations: GRU/ES/Maze, Mamba/PPO/Jumper, and GRU/PPO/Jumper, visualized with the first three canonical dimensions. Each figure is plotted based on 20,000 samples per task setting. (a) Color represents the distance of agent trajectories; (b) Color represents the orientation angle of agent trajectories.}
    \label{fig:grad-cues}
\end{figure*}

We measure the alignment between the above behavior and the neural dynamics representation using Canonical Correlation Analysis (CCA)~\citep{Hotelling1936, Raghu2017}, a widely used technique in analyzing correlation structure between two datasets in multivariate statistics (see Appendix~\ref{app:neu_beh_cca} for details).
When plotting the mapped canonical variates of the behavior-neural representations using CCA, we observe a remarkable \textit{structural isomorphism} between these two domains. As illustrated in Figure~\ref{fig:cca-overview}(b), for the maze navigation task, its behavior manifold (top figure) shares almost the same geometry with the neural dynamics manifold of limit cycles (bottom figure). Moreover, such correspondence even carries over to the semantic meaning of behaviors. For example, the inner regions of both the behavior and the neural dynamics manifolds contain mostly short-distance trajectories, whereas the outer regions contain more long-distance trajectories. Even more surprisingly, as shown in Figure~\ref{fig:cca-overview}(b), the residing locations of behaviors BF1-4 with different trajectory lengths and orientations in the behavior manifold are almost identical to the residing locations of their limit cycles in the neural dynamics manifold!

Such a high-level correspondence not only exists in the maze navigation task, but also exists widely for recurrent policies trained in other task families, training methods, and model architectures, as shown in Figure~\ref{fig:grad-cues} in the main text and more results in Appendix~\ref{app:cca_structures}. In these additional results, we also find a highly correlated geometric correspondence between behavioral and neural dynamics manifolds. Such a correspondence again carries rich and relational behavioral semantics, not only encoding the distance of agent trajectories, but also other behavioral attributes like the orientation angle of trajectories. This confirms that the neural dynamics topology preserves the metric properties of the physical behavior.
All these results demonstrate a high level of structural isomorphism, that the \textbf{manifold of neural limit cycles preserves the relational geometry of the physical behaviors, mapping similar trajectories to topologically adjacent attractors}. This might partly explain why recurrent policies often have better task generalization capabilities in meta-RL tasks~\citep{Duan2016RL2,Wang2016LearningToRL}. As the recurrent agent's underlying neural dynamics can encode continuous and relational structures of behaviors, it can smoothly adjust its strategy by traversing a continuous neural dynamics manifold, allowing the policy to interpolate the structured geometry to address novel task instances or non-stationary environments.

\begin{figure}[t]
  \centering
  \begin{minipage}{0.5\textwidth}
    \centering
    \includegraphics[width=\textwidth]{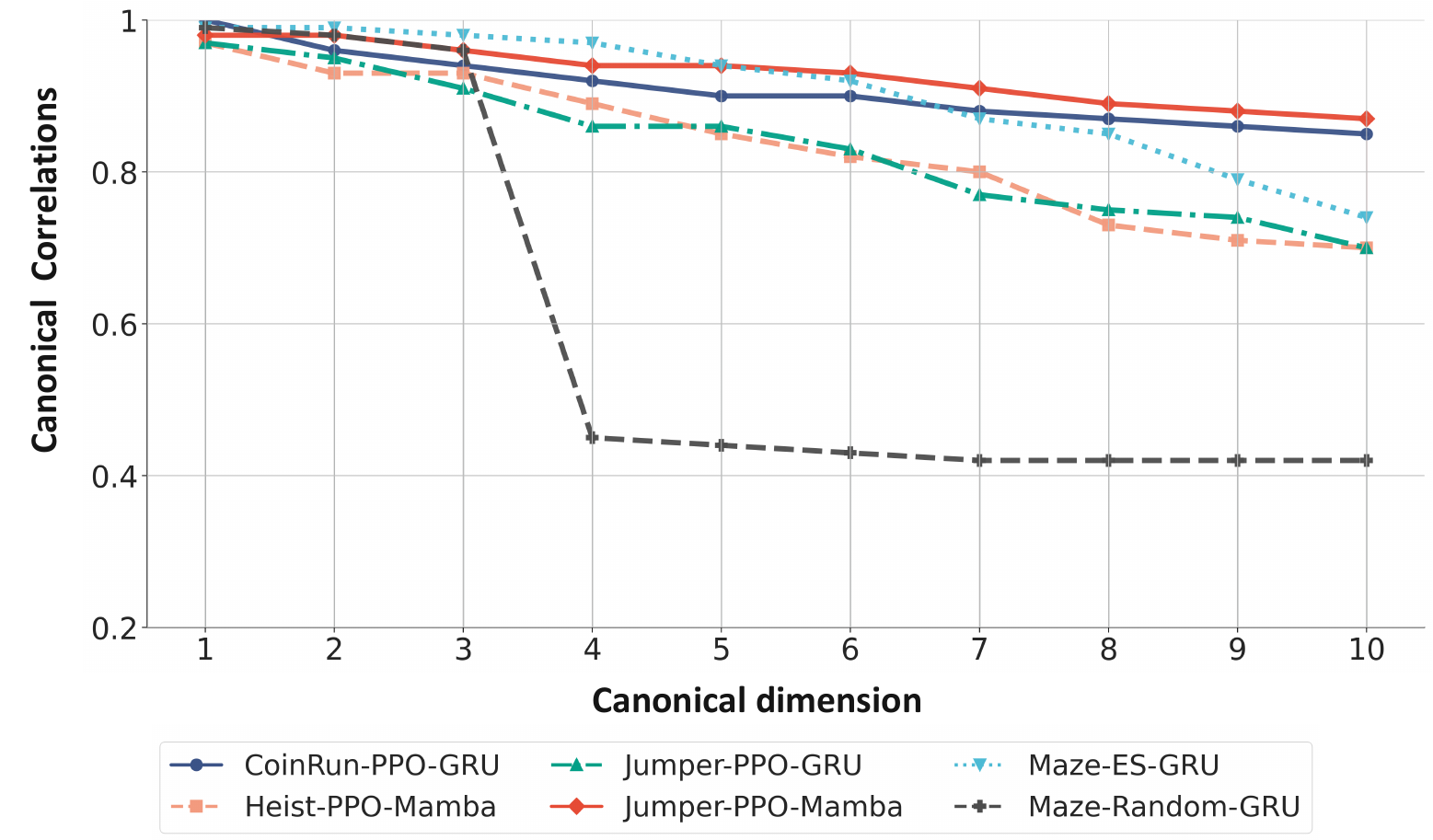}
    \caption{\textbf{Quantitative universality of high-dimensional alignment.} 
    Canonical correlation spectra for the top 10 modes. Optimized agents (solid lines) sustain correlations $>0.7$ across 10 dimensions, whereas a randomized baseline (dashed line) decays rapidly, indicating that the alignment is a learned property.}
    \label{fig:cca-spectrum}
  \end{minipage}
\end{figure}

To further confirm the strength of canonical correlation between the behavior and neural dynamics domains, we expand our CCA-based analysis to higher canonical dimensions, instead of only inspecting the visualized patterns on the first 2-3 canonical dimensions. The results are presented in Figure~\ref{fig:cca-spectrum}. Across all experimental settings—including Maze navigation and Procgen tasks trained via PPO or ES, using both GRU and Mamba—the canonical correlations between behavioral structure and neural dynamics remain high ($>0.7$) for all 10 dimensions.
To verify that this high-dimensional alignment is not an artifact of the state space dimensionality or action constraints, we conducted a random control experiment (see Appendix~\ref{app:control_experiment} for details). A randomized network, constrained to output the exact same action sequences ($N=10,000$ samples), exhibits a spectral collapse after the first 3 dimensions. This contrast indicates that the observed 10+ dimensional alignment is a specific result of the learned recurrent dynamics.

Moreover, we conduct counterfactual injection tests (see Appendix~\ref{app:counterfactual}) to rigorously verify that this behavioral-neural dynamics alignment is causal, rather than epiphenomenal. By surgically manipulating the agent's internal state—projecting optimal limit-cycle states into the canonical space and selectively lesioning specific dimensions—we test the functional role of the aligned geometry. As detailed in the convergence time histograms in Figure~\ref{fig:counterfactual_histograms} (Appendix), we demonstrate that the top CCA-aligned dimensions are both \textit{necessary} and \textit{sufficient} for optimal control. Specifically, preserving \textit{only} the top correlated modes (while randomizing all other neural dimensions) enables the agent to bypass exploration and immediately execute the optimal trajectory, shifting the convergence peak from $\sim$250 to $\sim$50 steps. Conversely, selectively randomizing just these top dimensions, even while keeping the vast majority of the neural state intact, completely abolishes this advantage, reverting the agent to baseline exploration performance. This confirms that the navigation-critical information is compressed almost entirely within the high-CCA subspace, establishing the aligned geometry as the functional basis of the agent's adaptive behavior.

\section{Conclusion}

We demonstrate that optimized recurrent neural policies in episodic settings consistently converge to attracting limit cycles. We identify Periodically-Kicked Drive (PKD) as a potential explanation for this phenomenon.
Crucially, we reveal that the intrinsic geometry of these neural limit cycles preserves the relational structure information of the physical behavior, which holds regardless of the task family, training methods, and recurrent architectures. These establish the neural dynamics manifold of limit cycles as the functional basis of adaptive behavior. 
Our findings offer new perspectives to explain many nice properties of recurrent neural policies, such as their superior robustness and adaptation/generalization capabilities.

\section{Discussion}

\subsection{Biological Connections and Geometric Universality}
Our results are consistent with recent findings in biological motor control. 
\citet{safaie2023preserved} reported that neural latent dynamics are preserved across distinct individuals within the same species performing similar motor behaviors, despite variations in individual neural circuitry. While that study hypothesized that this preservation arises from evolutionary constraints on circuit development, our findings suggest a complementary explanation based on geometric constraints.

We propose that the observed cross-individual alignment is an emergent property of solving a shared physical task. Since the physical environment imposes a specific relational geometry on the space of feasible behaviors, and optimized neural manifolds structurally mirror this behavioral geometry (structural isomorphism), the neural manifolds of distinct agents optimizing for the same task tend to converge to a similar geometric structure. Consequently, the "shared neural landscape" observed by \citet{safaie2023preserved} may naturally arise whenever recurrent systems—biological or artificial—optimize policies for the same physical reality.

\subsection{Limitations and Future Directions}
While this work establishes a correspondence between neural limit cycles and behavioral motifs, several aspects warrant further investigation.

First, our analysis focused on episodic tasks where recurrent policies converge to limit cycles. However, not all Reinforcement Learning (RL) settings induce explicit periodicity. In non-episodic or continuous control tasks, recurrent policies might encode optimal strategies via different dynamical regimes, such as stable fixed points or chaotic attractors. We hypothesize that the principle of structural isomorphism extends to these alternative topological structures, though verification in non-periodic settings remains to be performed.

Second, the Behavioral Potential Field (BPF) metric was implemented here primarily for 2D navigation tasks. Theoretically, BPFs can be generalized to high-dimensional state spaces, such as the joint-space trajectories of robotic manipulators. However, computing BPFs in high-dimensional spaces introduces significant computational scalability challenges. Developing efficient representations to quantify behavioral similarity in high-dimensional continuous control domains is a necessary step for generalizing these findings to complex robotic systems.

Finally, while this study characterizes the stable limit cycles representing the converged policy, the transient dynamics of the optimization process remain to be explored. A critical next step is to analyze how the recurrent policy utilizes reward signals to drive the neural state from initialization toward these stable attractors. Applying dynamical systems tools to these transient trajectories could help elucidate the implicit "in-context" optimization algorithms implemented by the RNN, which are meta-learned from data and may differ from hand-engineered update rules.

\section*{Acknowledgement}

We thank Jia Liu from the Department of Psychology and Cognitive Sciences, Tsinghua University, for his support of this work and financial support provided by the department.
We are grateful to Xuena Wang, Meng Lu, Mingli Yuan, and Hengshuai Yao for their valuable discussions and insightful feedback.
This work was also supported by funding from Sapient.




\bibliographystyle{icml2026}
\bibliography{example_paper}

\newpage
\appendix
\onecolumn

\section{Experiment Setup}

To evaluate the universality of the observed dynamical structures, we conduct experiments across two distinct regimes: (1) a controlled, open-ended grid world trained via neuro-evolution, and (2) high-dimensional visual control tasks trained via gradient-based reinforcement learning.

\subsection{Open-Ended Maze Navigation (Evolution Strategies)}
\label{app:training}

\textbf{Task Environment.} We implement an open-ended navigation task on a $10 \times 10$ discrete grid. To ensure diversity, mazes are procedurally generated by assigning obstacles to cells with a 30\% probability, followed by a connectivity check (via connected component analysis using OpenCV) to discard unsolvable layouts. This results in a large combinatorial search space ($2^{100}$ potential configurations), preventing the agent from relying on memorization. 

The agent operates under partial observability, receiving input only from a $3 \times 3$ local window centered on its current position. This information bottleneck requires the agent to integrate information over time to build an internal model of the maze structure. The action space consists of four discrete movements: \textit{up, down, left, right}.

\textbf{Recurrent Architecture.} The agent is parameterized by a Gated Recurrent Unit (GRU). At each time step $t$, the network receives a concatenated input vector $[o_t, h_{t-1}]$, where $o_t$ is the flattened local perception and $h_{t-1}$ is the hidden state from the previous step. A linear readout layer maps the updated hidden state $h_t$ to a softmax distribution over the four actions.

\textbf{Meta-RL Trial Structure.} We adopt an $RL^2$-style trial protocol to foster rapid adaptation. A single trial consists of multiple episodes within the same fixed maze. Crucially, when the agent reaches the goal (or times out), its physical position is reset to the start, but its \textit{hidden state is preserved}. This continuity of memory allows the agent to exploit information gathered in early episodes to optimize trajectory efficiency in later episodes.

\textbf{Training via Evolution Strategies.} We optimize the policy using OpenAI Evolution Strategies. The objective is to minimize the path length of the \textit{final} episode in a trial, incentivizing the agent solve the maze. The fitness evaluation logic is detailed in Algorithm~\ref{alg:nes_maze}.

\begin{algorithm}[h]
\caption{Meta-Training Fitness Evaluation}
\label{alg:nes_maze}
\begin{algorithmic}[1]
\STATE {\bfseries Input:} Population of perturbed parameters $\{\theta_i\}_{i=1}^N$, Maze Generator $\mathcal{M}$
\FOR{each candidate $\theta_i$ in parallel}
    \STATE Sample maze $M \sim \mathcal{M}$
    \STATE Initialize hidden state $h_0 = \mathbf{0}$, position $p_0 = \text{Start}$
    \STATE \textit{// Execute Trial (Sequence of Episodes)}
    \WHILE{Trial Budget not exceeded}
        \STATE Run episode until Goal or Timeout
        \IF{Goal Reached}
            \STATE $L_{last} \leftarrow$ Current episode path length
            \STATE Reset $p \to \text{Start}$ \textbf{only} (keep $h$)
        \ELSE
            \STATE Reset $p \to \text{Start}$ \textbf{and} $h \to \mathbf{0}$ (soft failure reset)
        \ENDIF
    \ENDWHILE
    \STATE {\bfseries Return} Fitness $F_i = \text{CentralRankTransform}(L_{last})$ 
    \STATE \textit{// Higher rank assigned to shorter final path lengths}
\ENDFOR
\end{algorithmic}
\end{algorithm}

\subsection{High-Dimensional Visual Control (Procgen)}

\textbf{Task Environment.} We evaluate our hypothesis on the Procgen Benchmark, a suite of procedurally generated arcade-style environments that feature high-dimensional visual observations ($64 \times 64 \times 3$ RGB) and diverse level structures. We specifically focus on tasks such as \textit{CoinRun}, \textit{Jumper}, and \textit{Heist}, which require the agent to navigate complex spatial layouts and retain memory of key environmental features. To test generalization, we use a fixed set of seeds for training levels and a disjoint set of seeds for evaluation levels.

\textbf{Recurrent Architecture.} The agent architecture consists of a convolutional encoder followed by a recurrent dynamics model. The encoder processes the visual input into a compact latent embedding. This embedding is then fed into the recurrent core—either a standard Gated Recurrent Unit (GRU) or a Mamba-based State Space Model—which maintains the agent's hidden state $h_t$. The recurrent output is projected by two separate linear heads: a policy head producing a categorical distribution over discrete actions, and a value head estimating the state value function.

\textbf{Meta-RL Trial Structure.} Consistent with the $RL^2$ framework, we organize training into ``trials.'' Each trial consists of a fixed number of episodes (e.g., $K=2 \sim 5$) played on the same unique level seed. Crucially, the recurrent hidden state is preserved across episode boundaries within a trial, and is only reset to zero when a new trial (i.e., a new level seed) begins. This setup forces the agent to use its activation dynamics to adapt to the specific level layout over the course of the trial.

\textbf{Training via PPO.} We train the recurrent policies using Proximal Policy Optimization (PPO) with Generalized Advantage Estimation (GAE). We use the Adam optimizer with a learning rate of $5 \times 10^{-4}$. The training process runs on parallel environments to collect experience, with gradients averaged across mini-batches. To encourage exploration and prevent premature convergence, we include an entropy regularization term in the objective. The optimization loop allows the agent to continuously refine its internal dynamics to maximize the cumulative reward over the entire trial duration.

\section{Hybrid Dynamical Systems}
\label{app:hds}

\textbf{System Modeling.} To understand the causal coupling between the agent and the environment, we formalize the interaction as a closed-loop \textbf{Hybrid Dynamical System (HDS)}. We define the joint state space as $\mathcal{X} = \mathcal{S} \times \mathcal{H}$, where $\mathcal{S}$ represents the discrete environmental states (e.g., maze coordinates, obstacle configurations) and $\mathcal{H} \subseteq \mathbb{R}^K$ denotes the continuous neural hidden states.

The time-evolution of this joint system can be condensed into two coupled recurrence relations representing the agent and the environment respectively:

\[
\begin{cases}
h_{t+1} = f_{\theta}(\mathrm{obs}(s_t),h_t), & \text{(Agent Memory Dynamics)} \\
s_{t+1} = \mathcal{T}(s_t, \pi_{\theta}(s_t, h_t)). & \text{(Physical State Evolution)}
\end{cases}
\]

Here, the first equation governs the agent's internal evolution, where the recurrent function $f_{\theta}$ updates the memory based on the partial observation derived from the current physical state $s_t$. The second equation governs the physical evolution, where the environment's transition operator $\mathcal{T}$ updates the physical state based on the action sampled directly from the policy $\pi_{\theta}(s_t, h_t)$.

This formulation establishes that the agent-environment loop constitutes a unified dynamical entity. By tracking the evolution of the joint state $x_t = (s_t, h_t)$ in the product space $\mathcal{X}$, we address the bidirectional coupling characteristic of embodied intelligence: we no longer view the network as a passive responder to environmental stimuli, but as an active component in a bidirectional causal loop where neural states drive physical transitions and vice versa.

\section{Empirical Verification of Local Contractivity in Neural State Space}
\label{app:empirical_contractivity}

Before formally deriving the Periodically-Kicked Drive (PKD) mechanism, we present empirical evidence verifying a necessary condition for limit cycle formation: the existence of a \textbf{contractive region} in the neural state space. We investigate whether the stable trajectories generated by a trained agent are confined within a region where the intrinsic neural dynamics exhibit volume contraction.

\subsection{Local Linear Dynamical System (LDS) Analysis}

To quantify the local stability characteristics along the neural trajectory, we analyze the agent's recurrent dynamics as a continuous-time vector field. For a recurrent unit (e.g., RNN) with hidden state $h \in \mathbb{R}^n$, the intrinsic update dynamics (ignoring external inputs for local stability analysis) can be approximated by the vector field:
\begin{equation}
    \dot{h} \approx F(h) = \tanh(W_{rec} h + b) - h
\end{equation}
where $W_{rec}$ represents the recurrent weight matrix and $b$ is the bias. An equilibrium or stable manifold requires the divergence of this field to be negative.

We employ Local Linearization to analyze the stability at each point $h_t$ along the recorded trajectory. By performing a first-order Taylor expansion around $h_t$, we approximate the non-linear system as a Local Linear Dynamical System (LDS):
\begin{equation}
    F(h) \approx F(h_t) + J_F(h_t) \cdot (h - h_t)
\end{equation}
where $J_F(h_t) = \frac{\partial F}{\partial h}\big|_{h_t}$ is the Jacobian matrix evaluated at the current state.

The stability of the system at point $h_t$ is determined by the eigenspectrum of the Jacobian $J_F(h_t)$. Specifically, if the real parts of all eigenvalues $\lambda_i$ are negative:
\begin{equation}
    \max_i \text{Re}(\lambda_i(J_F(h_t))) < 0
\end{equation}
then the local dynamics are contractive, meaning that small perturbations in the state space will decay over time, pulling the system back towards the attractor.

\subsection{Experimental Results}

We conducted this analysis on a GRU agent fully trained on a POMDP Grid Maze task. Once the agent stabilized into a repetitive optimal strategy, we captured a sequence of 50 neural states along its limit cycle.

Figure~\ref{fig:eigen_spectrum} visualizes the results of this stability check:
\begin{itemize}
    \item \textbf{Behavioral \& Neural Cycles (Top):} The agent performs a consistent loop in the maze (Top Left), which corresponds to a distinct closed orbit in the PCA projection of the neural states (Top Right).
    \item \textbf{Eigenvalue Spectra (Bottom):} For every single time step ($T1$ to $T50$) along this orbit, we computed the eigenvalues of the local Jacobian. The bar charts display the sorted real parts of these eigenvalues.
\end{itemize}

\textbf{Conclusion:} As evidenced by the uniformly blue histograms, the maximum real eigenvalue is strictly negative for \textbf{every point} along the trajectory. This confirms that the entire operational orbit of the agent lies within a locally contractive region of the state space. This consistent contractivity provides the dynamical basis for the formation of stable limit cycles, as it ensures that the neural state is constantly being converging towards the attractor, counteracting any expansive effects from environmental noise or inputs.

\begin{figure*}[h]
    \centering
    \includegraphics[width=0.7\textwidth]{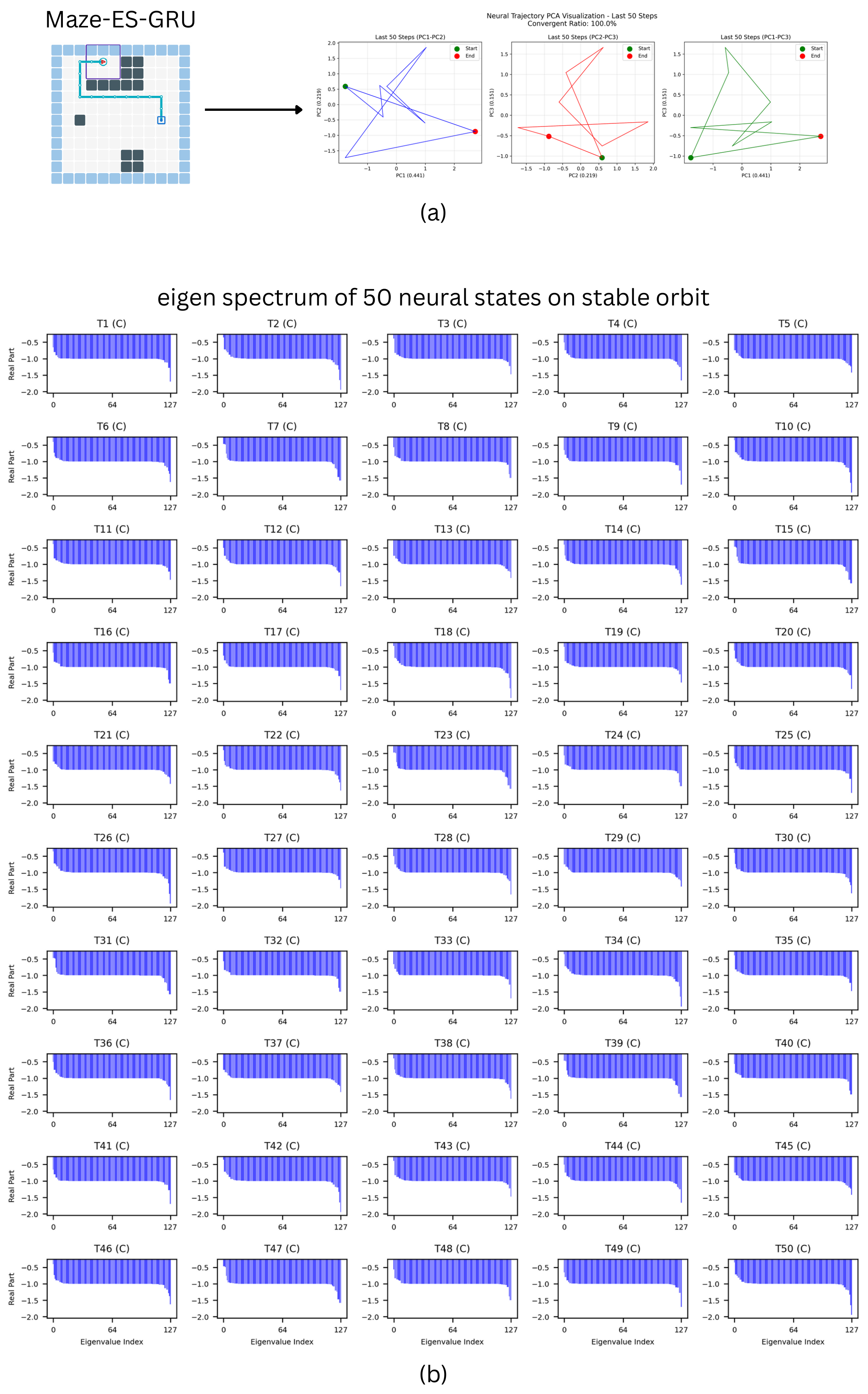}
    \caption{\textbf{Empirical verification of contractive dynamics along the neural limit cycle.}
    \textbf{(Top Left)} The physical trajectory of a converged GRU agent in a POMDP Maze task.
    \textbf{(Top Right)} PCA visualization of the neural hidden states for the last 50 steps, revealing a low-dimensional closed orbit structure.
    \textbf{(Bottom)} Detailed eigenvalue spectra for the local Jacobian matrix at each of the 50 time steps along the cycle. The y-axis represents the real part of the eigenvalues.
    Across all time steps ($T1-T50$), the spectral components are entirely negative (blue bars), with no positive components (which would appear as red bars above the dashed line). This indicates that the local dynamics are strictly contractive ($\text{Re}(\lambda) < 0$) throughout the entire execution of the behavior, providing the necessary dynamical condition for stable limit cycle maintenance.}
    \label{fig:eigen_spectrum}
\end{figure*}

\clearpage

\section{Theoretical Analysis of Stability and Limit Cycles}
\label{app:proof_limit_cycle}

\paragraph{Conditioned Neural Subsystem.}
We analyze the stability of the recurrent policy by focusing on its neural subsystem,
while conditioning on the effective input sequence induced by the closed-loop
agent--environment interaction described in Section~\ref{app:hds}.
Specifically, the physical environment state $s_t$ influences the recurrent dynamics
only through the observation and encoding pipeline, yielding an effective input
sequence $\{u_t\}$ to the recurrent policy.
\emph{Under this conditioned viewpoint}, we provide a formal proof that a periodically driven,
dissipative recurrent policy induces a unique, globally attracting stable limit cycle.

\subsection{Setup and Definitions}

We first formalize the properties of the periodic input drive and the dissipative nature of the policy's recurrent dynamics.

\begin{definition}[Periodic Input Drive]
\label{def:periodic_input}
Let $\{u_t\}_{t\ge 0}$ denote the sequence of effective inputs to the recurrent policy,
induced by the closed-loop agent--environment dynamics through the observation
and input-encoding pipeline. We say the drive is $T$-periodic if:
\begin{equation}
u_{t+T} = u_t, \qquad \forall t \ge 0.
\end{equation}
\end{definition}

\begin{definition}[Contractive Region]
\label{def:contractive_region}
A subset $\mathcal{C} \subset \mathbb{R}^n$ is called a \emph{contractive region} if it is forward-invariant under the policy dynamics and the recurrent update is strictly contractive within $\mathcal{C}$. Specifically:
\begin{enumerate}
    \item (\textbf{Forward invariance}) For any $h \in \mathcal{C}$ and any admissible input $u$, we have
    \[
        f_\theta(u, h) \in \mathcal{C}.
    \]
    \item (\textbf{Local contractivity}) There exists a norm $\|\cdot\|$ and a constant $0 < \lambda < 1$ such that for all $h, h' \in \mathcal{C}$ and any input $u$,
    \[
        \|f_\theta(u, h) - f_\theta(u, h')\| \le \lambda \|h - h'\|.
    \]
\end{enumerate}
\end{definition}

\begin{definition}[Dissipative Policy Dynamics (Restricted)]
\label{def:dissipative}
Let $h_t \in \mathbb{R}^n$ be the policy memory state evolving according to $h_{t+1} = f_\theta(u_t, h_t)$. 
We assume that there exists a contractive region $\mathcal{C} \subset \mathbb{R}^n$ such that the dynamics are strictly contractive within $\mathcal{C}$ in the sense of Definition~\ref{def:contractive_region}.
\end{definition}

To analyze the long-term behavior over full episodes, we introduce the stroboscopic map.

\begin{definition}[Stroboscopic Map]
Fix a $T$-periodic effective input sequence $\{u_t\}_{t=0}^{T-1}$,
induced by a converged closed-loop behavioral regime of the hybrid system. Let $F_t(\cdot) := f(u_t, \cdot)$. The stroboscopic map $S: \mathbb{R}^n \to \mathbb{R}^n$ represents the state evolution over one complete period $T$:
\begin{equation}
S := F_{T-1} \circ F_{T-2} \circ \cdots \circ F_0.
\end{equation}
A fixed point $h^*$ of $S$ (i.e., $S(h^*) = h^*$) corresponds to the initial state of a $T$-periodic trajectory. We assume that the trajectory $\{h_t\}$ remains entirely within the contractive region $\mathcal{C}$.
\end{definition}

\subsection{Proof of Existence and Uniqueness}

\begin{lemma}[Contraction of the Stroboscopic Map]
Assume the policy dynamics are contractive within a region $\mathcal{C}$ and that the trajectory remains in $\mathcal{C}$. 
Then the stroboscopic map $S$ restricted to $\mathcal{C}$ is a contraction mapping with Lipschitz constant $\lambda^T$. That is:
\begin{equation}
\|S(h_1) - S(h_2)\| \le \lambda^T \|h_1 - h_2\|, \qquad \forall h_1, h_2 \in \mathcal{C}.
\end{equation}
\end{lemma}

\begin{proof}
By Definition \ref{def:dissipative}, each single-step map $F_t$ is $\lambda$-Lipschitz. Since $S$ is the composition of $T$ such maps, the Lipschitz constant of $S$ is the product of the individual constants. Thus, $L_S \le \prod_{t=0}^{T-1} \lambda = \lambda^T$. Since $0 < \lambda < 1$, we have $\lambda^T < 1$.
\end{proof}

\begin{theorem}[Existence of a Stable Limit Cycle within a Contractive Region]
\label{thm:main_proof}
Given a $T$-periodic input and dissipative policy dynamics within a contractive region $\mathcal{C}$, there exists a unique $T$-periodic trajectory entirely contained in $\mathcal{C}$. 
Furthermore, for any initial state $h_0 \in \mathcal{C}$, the policy trajectory converges exponentially to this periodic orbit.
\end{theorem}

\begin{proof}
\textbf{Existence and Uniqueness:} Since $S$ is a contraction mapping on the Banach space $\mathbb{R}^n$ (Lemma 1), by the Banach Fixed-Point Theorem, $S$ admits a unique fixed point $h^*$. This fixed point generates the periodic orbit $h_{t+1}^* = f(u_t, h_t^*)$ with $h_0^* = h^*$, satisfying $h_T^* = S(h^*) = h^*$.

\textbf{Global Convergence:} The Banach Fixed-Point Theorem further guarantees that for any initial $h_0$, the sequence of states at the start of each episode, $h_{kT} = S^k(h_0)$, converges to $h^*$ as $k \to \infty$. specifically:
\begin{equation}
\|h_{kT} - h^*\| \le \lambda^{kT} \|h_0 - h^*\|.
\end{equation}
This implies exponential convergence of the dynamical system to the unique limit cycle.
\end{proof}

\section{Empirically Observed Limit Cycles}\label{app:lc}

Figure~\ref{fig:limit_cycles_more} provides a qualitative survey of empirically observed limit cycles in Procgen, spanning 8 environments, 3 recurrent architectures (RNN, GRU, Mamba), and 2 training pipelines (PPO, evolution strategies).
Each subpanel corresponds to a fixed task--method--architecture configuration and contains three representative procedural seeds; for each seed we show the \emph{initial observation} (top strip) and the resulting PCA-projected recurrent hidden-state trajectory over one converged episode (bottom).

Across all settings, the projected trajectories exhibit consistent signatures of a stable periodic attractor in the closed-loop agent--environment dynamics: (i) a simple closed-orbit topology (clear loops rather than space-filling curves), (ii) smooth, non-self-intersecting paths indicative of evolution on a low-dimensional manifold, and (iii) periods commensurate with the episode length.
Moreover, the high-dimensional closure residual $\|\mathbf{h}_T - \mathbf{h}_0\|_2$ concentrates around $\mathcal{O}(10^{-5})$, supporting that the observed loops reflect genuine periodic orbits rather than projection artifacts (Appendix~\ref{app:acf}).

These examples are intended as qualitative support for the broader empirical regularity: after convergence, recurrent policies frequently settle into stable cyclic regimes, consistent with the theoretical prediction that dissipative dynamics under periodic forcing promote convergence to stable periodic orbits (Appendix~\ref{app:proof_limit_cycle}).

\begin{figure*}[t]
    \centering
    \includegraphics[width=0.9\textwidth]{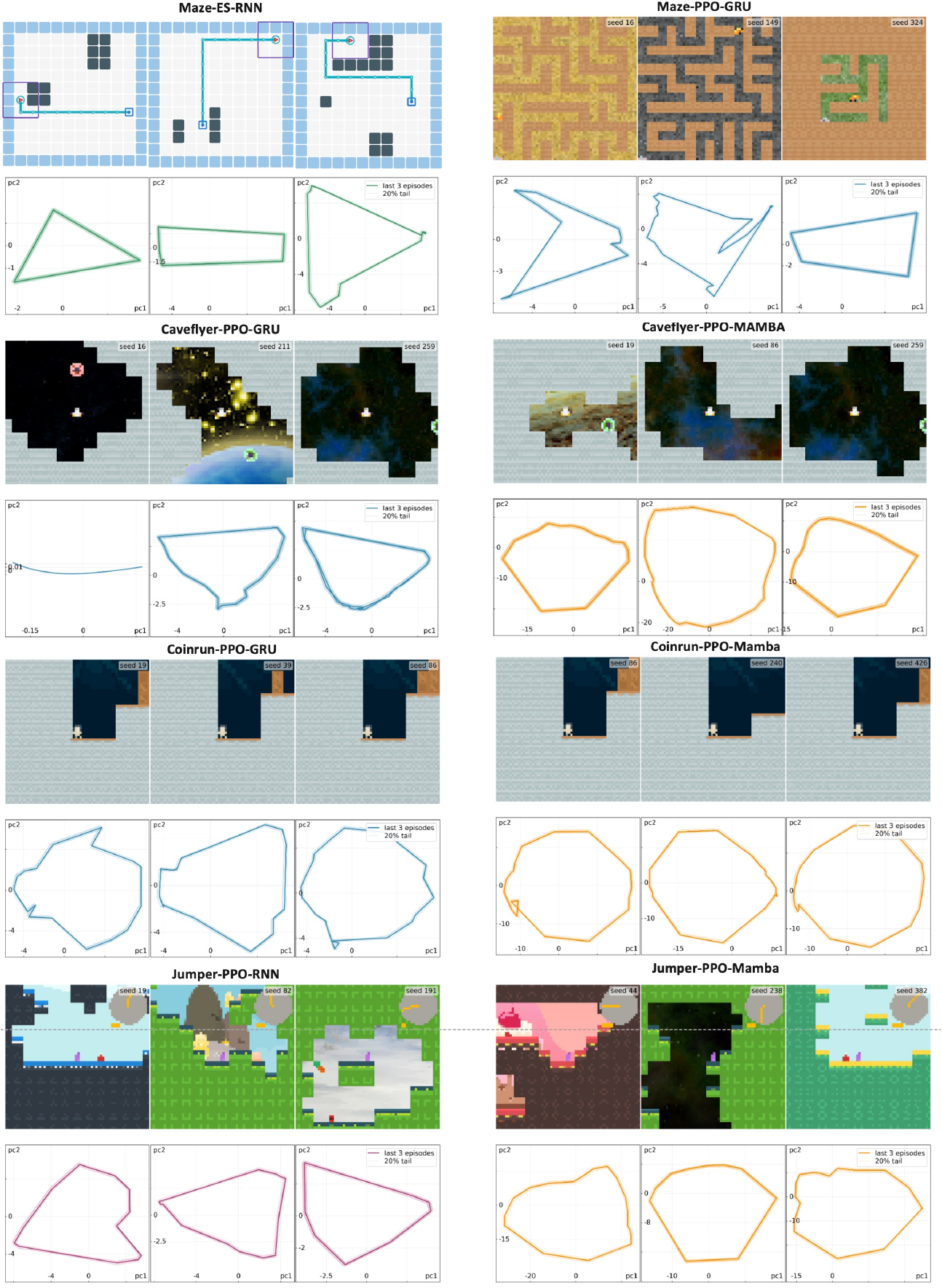}
    \caption{\textbf{Gallery of empirically observed limit cycles.}
    Each subpanel corresponds to a fixed (task, training pipeline, architecture) configuration and shows three representative procedural seeds.
    For each seed, the \textbf{top strip} shows the initial observation and the \textbf{bottom} shows the PCA-projected recurrent hidden-state trajectory over a converged episode.
    Across environments and training setups, trajectories consistently form stable, smooth closed loops, with orbit geometry varying by task but remaining topologically simple.}
    \label{fig:limit_cycles_more}
\end{figure*}

\begin{figure*}[t]
    \centering
    \includegraphics[width=\textwidth]{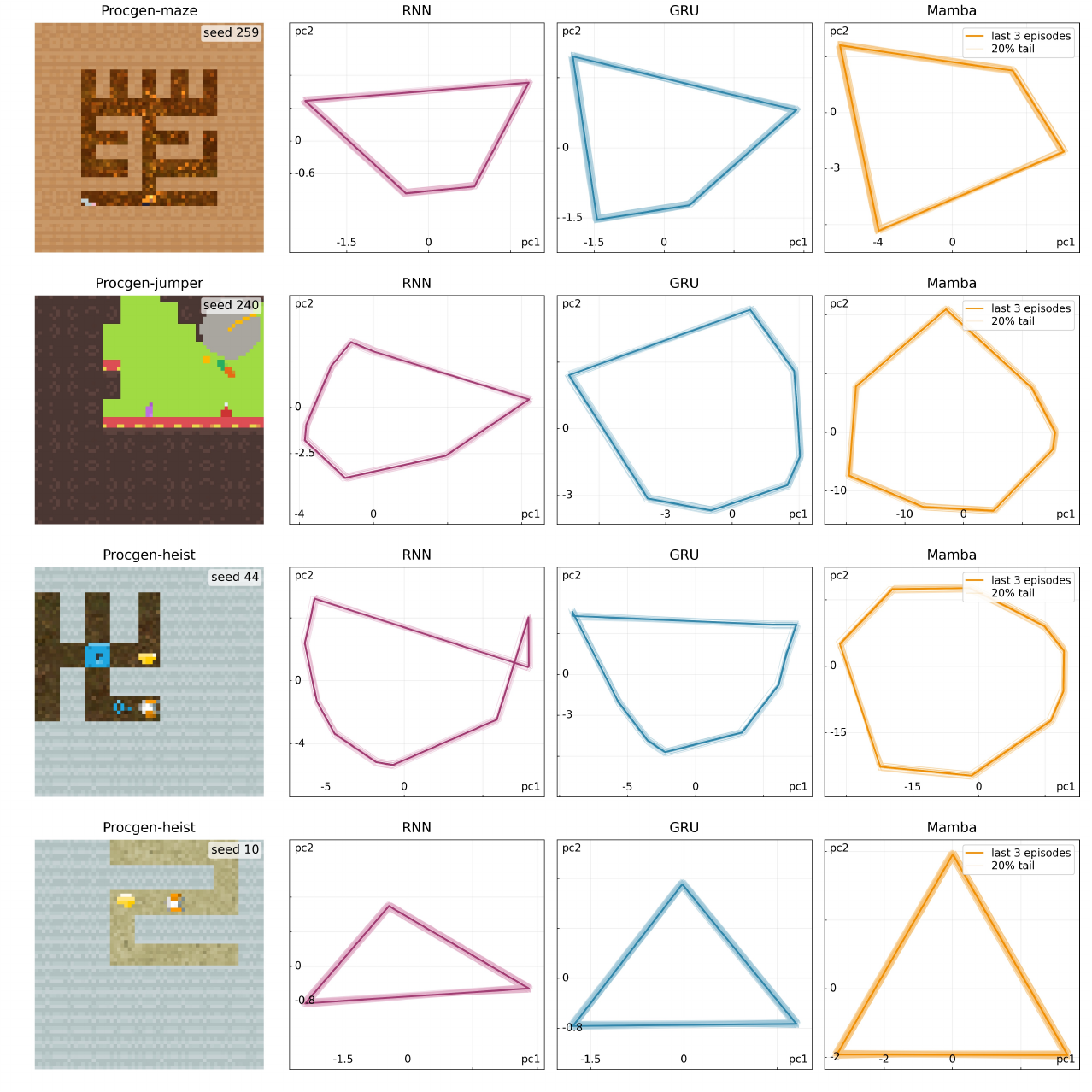}
    \caption{\textbf{Cross-architecture consistency under fixed environment and seed.}
    We fix the Procgen environment and the procedural seed, and compare the converged recurrent dynamics across different architectures.
    As in Figure~\ref{fig:limit_cycles_more}, the top strip shows the seed's initial observation and the bottom shows the PCA-projected hidden-state trajectory over an episode.
    Despite architectural differences, the resulting limit-cycle geometry is notably similar, suggesting that near-optimal closed-loop solutions induce a task-instance-conditioned attractor whose coarse geometric structure is largely architecture-invariant.
    This qualitative observation is consistent with our quantitative neural--behavioral alignment analysis (Appendix~\ref{app:cca_structures}), which reveals high-dimensional correspondences between neural manifolds and behavioral geometry across settings.}
    \label{fig:same_seed_dif_arch}
\end{figure*}

\clearpage

\section{Robustness of Limit Cycles under Hidden-State Perturbations}
\label{app:perturb}

To extend the illustrative example in Figure~\ref{fig:perturbation}, we evaluate limit-cycle stability under perturbations across many task instances (procedural seeds) and recurrent policies in both \emph{maze navigation} and \emph{Procgen}.
For each trained policy $\pi_\theta$ (parameters $\theta$ fixed) and a fixed task instance, we first record a baseline hidden-state trajectory $\{h_t^{\mathrm{base}}\}_{t=0}^{T}$ in the converged regime, then inject an additive perturbation at a chosen time step $t^*$:
\begin{equation}
    \tilde{h}_{t^*} = h_{t^*}^{\mathrm{base}} + \epsilon\,\delta, \qquad \delta \sim \mathcal{N}(0, I).
\end{equation}
The perturbation is applied only to the recurrent hidden state; the environment dynamics and observation stream are otherwise unchanged.
We run several perturbation variants per instance (independent draws of $\delta$) and visualize baseline and perturbed trajectories in the same PCA projection used throughout the paper (PCs fitted on the baseline trajectory).

\begin{figure*}[t]
    \centering
    \IfFileExists{fig/perturb-more.pdf}{
        \includegraphics[width=\textwidth]{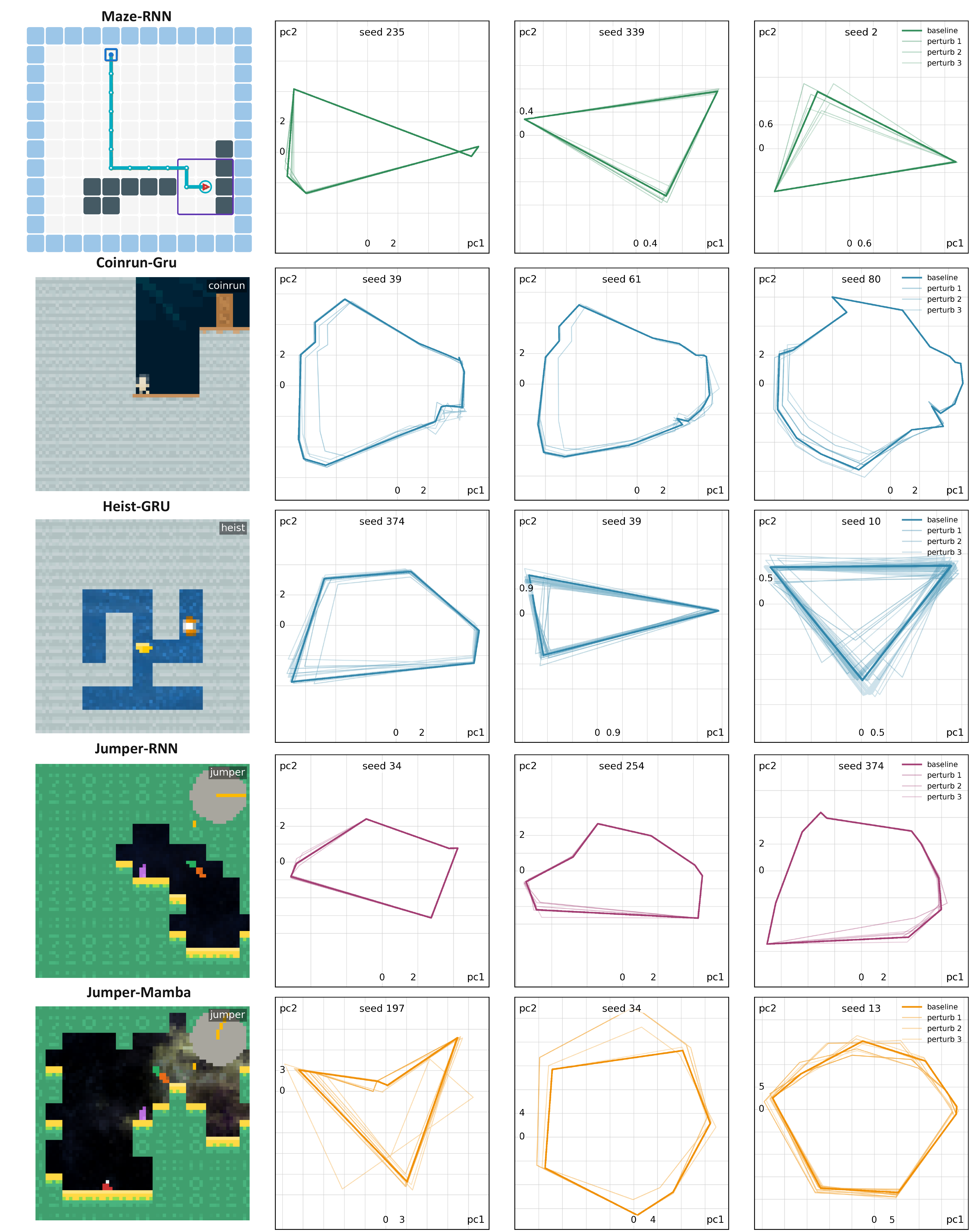}
    }{
        \fbox{\parbox{0.95\textwidth}{\centering Missing file: \texttt{fig/perturb-more.pdf}}}
    }
    \caption{Perturbation robustness across maze navigation and Procgen.
    Each panel shows PCA-projected hidden-state trajectories for a trained policy on a fixed task instance.
    The darker curve is the unperturbed baseline cycle; lighter curves are trajectories after hidden-state perturbations (policy weights fixed).
    Across a broad set of environments and seeds, perturbations induce only transient deviations and trajectories contract back to the same task-conditioned limit cycle, consistent with the contractive dynamics in Theorem~\ref{thm:main_proof}.}
    \label{fig:perturb_more}
\end{figure*}

\paragraph{Boundary cases}
We also observe counterexamples where perturbations lead to substantially longer transients before trajectories approach a stable cyclic regime.
Figure~\ref{fig:perturb_negative} highlights two such Procgen environments, \textit{Bossfight} and \textit{Bigfish}.
Compared to maze navigation, these games include richer non-stationary interaction dynamics (e.g., moving adversaries, projectiles, prey) whose evolution is not fully controlled by the agent and can dominate the visual stream; equivalently, the closed-loop system is driven by higher-dimensional exogenous degrees of freedom.
Moreover, successful behavior in \textit{Bossfight} and \textit{Bigfish} is often primarily reactive---requiring continual adjustment to the instantaneous configuration of enemies/projectiles/fish---which further weakens the emergence of a short, repeatable observation--action pattern (periodic forcing) and yields longer excursions in hidden-state space.
Nevertheless, partial re-stabilization is still visible for some seeds (e.g., Bossfight seed 34; Bigfish seed 122), suggesting that contractive basins may exist locally even when global re-convergence is slower.

\begin{figure*}[t]
    \centering
    \IfFileExists{fig/perturb-negative.pdf}{
        \includegraphics[width=\textwidth]{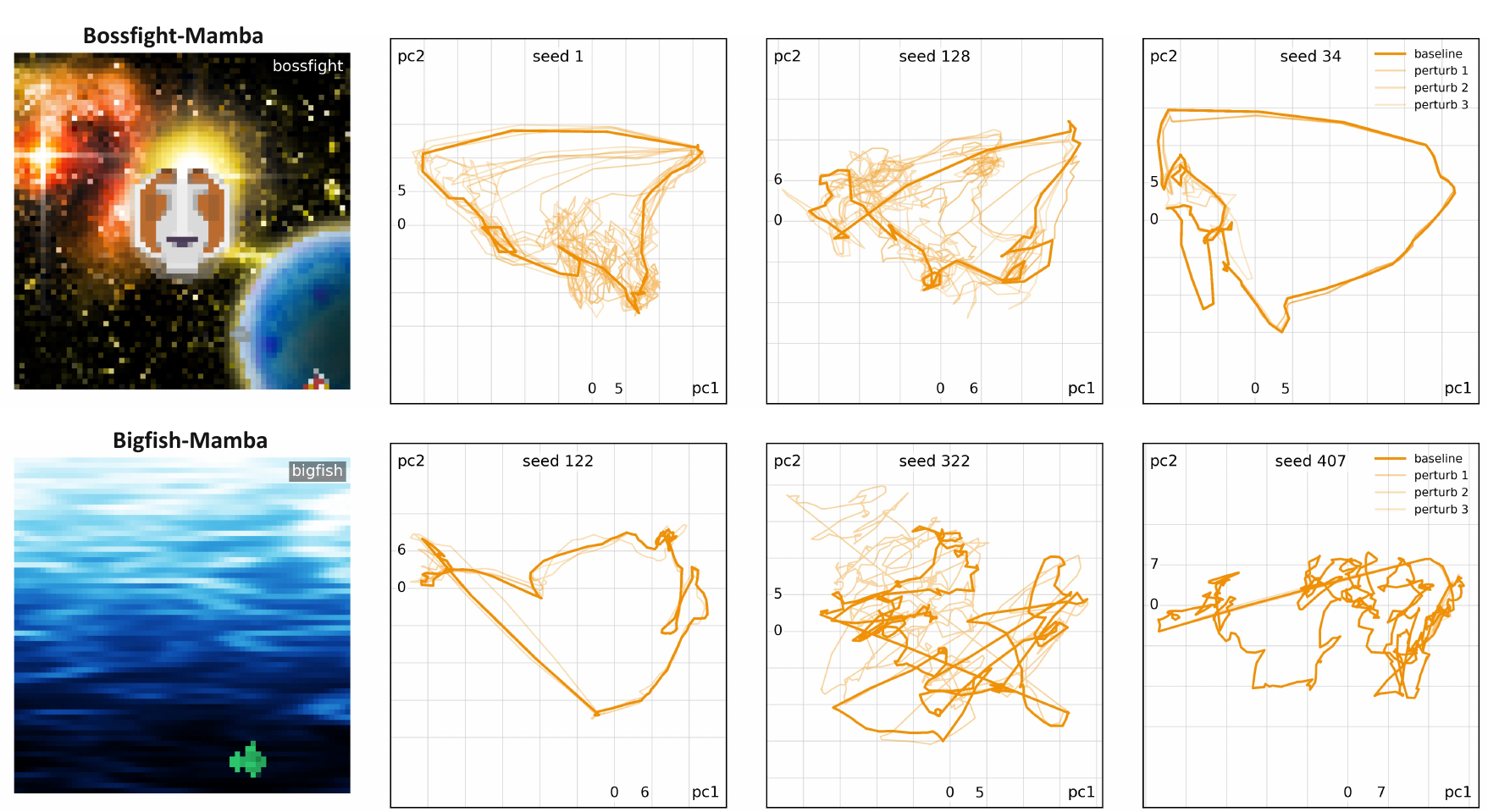}
    }{
        \fbox{\parbox{0.95\textwidth}{\centering Missing file: \texttt{fig/perturb-negative.pdf}}}
    }
    \caption{Extended transients under perturbation in Procgen.
    In \textit{Bossfight} and \textit{Bigfish}, hidden-state perturbations can induce prolonged excursions before trajectories approach a stable regime.
    These environments exhibit stronger non-stationarity due to moving entities and reactive interactions, which can disrupt the quasi-periodic structure that supports rapid re-convergence in maze navigation.
    Partial stabilization is still visible for some seeds (e.g., Bossfight seed 34; Bigfish seed 122).}
    \label{fig:perturb_negative}
\end{figure*}

\clearpage

\section{Limit Cycle Sampling via PKD and Action-Consistency Filter (ACF)}
\label{app:acf}

While the Periodically-Kicked Drive (PKD) mechanism effectively reveals the distribution of \textit{potential} limit cycles within the recurrent policy's state space, not all such cycles are functionally relevant. A periodic input stream might entrain the network into a stable "ghost orbit" that is mathematically stable under the forced input but generates actions inconsistent with the input itself. Such an orbit would break the causal loop in a real interaction.

To isolate the specific neural structure responsible for a given behavior, we employ the \textbf{Action-Consistency Filter (ACF)}. This procedure verifies the \textbf{closed-loop self-sustainment} condition: a valid neural limit cycle must produce the exact sequence of actions required to generate the observation stream that drives it.

The filtering pipeline proceeds as follows:
\begin{enumerate}
    \item \textbf{PKD Extraction:} We first record a ground-truth trajectory of length $T$ (the period) from the agent's interaction with the environment, extracting the cyclic observation sequence $\mathcal{O}_{seq}$ and the target action sequence $\mathcal{A}_{target}$.
    \item \textbf{Large-Scale Probing:} We initialize a large population of random hidden states ($N=100,000$) to broadly sample the basins of attraction in the neural state space.
    \item \textbf{Entrainment:} We apply the observation sequence $\mathcal{O}_{seq}$ as a repeated external drive to these states for a sufficient warmup period (e.g., 1,000 steps). Due to the contractive nature of the trained network, most trajectories will collapse onto stable manifolds.
    \item \textbf{Convergence \& Consistency Check:} After entrainment, we first discard trajectories that have not converged to a periodic orbit. To address concerns regarding high-dimensional closure, we explicitly measure the \textbf{Closure Error} $\delta = \|\mathbf{h}_T - \mathbf{h}_0\|_2$. We find that for retained candidates, $\delta$ is concentrated at the magnitude of $10^{-5}$, confirming that these are genuine periodic orbits in the high-dimensional space, not merely projection artifacts.
    \item \textbf{Filtering:} Finally, we check the readout constraint: does the neural state at each step of the cycle decode to the correct action in $\mathcal{A}_{target}$? Only cycles satisfying both convergence and action consistency are retained.
\end{enumerate}

The detailed algorithm is provided in Algorithm~\ref{alg:acf}.

\begin{algorithm}[h]
\caption{Action-Consistency Filter (ACF)}
\label{alg:acf}
\begin{algorithmic}[1]
\STATE {\bfseries Input:} Recurrent Policy $f_\theta$, Readout $\pi_\theta$
\STATE {\bfseries Input:} Ground Truth Obs Sequence $\mathcal{O} = \{o_1, \dots, o_T\}$, Target Actions $\mathcal{A} = \{a_1, \dots, a_T\}$
\STATE {\bfseries Hyperparameters:} Number of seeds $N=100000$, Warmup steps $K=1000$, Convergence thresh $\epsilon$
\STATE {\bfseries Output:} Set of Validated Limit Cycles $\mathbf{C}_{valid}$

\STATE Initialize $\mathbf{C}_{valid} \leftarrow \emptyset$
\STATE Generate $N$ random initial hidden states $\{h^{(i)}_0\}_{i=1}^N \sim \mathcal{N}(0, I)$

\FOR{each candidate state $h^{(i)}_0$ in parallel}
    \STATE $h \leftarrow h^{(i)}_0$
    \STATE \textit{// Phase 1: Entrainment (Apply PKD)}
    \FOR{step $k = 1$ to $K$}
        \STATE $t \leftarrow (k-1) \pmod T + 1$
        \STATE $h \leftarrow f_\theta(h, o_t)$
    \ENDFOR
    
    \STATE \textit{// Phase 2: Capture Cycle \& Check Convergence}
    \STATE $h_{start} \leftarrow h$
    \STATE $\text{Cycle} \leftarrow []$
    \STATE $\text{Actions} \leftarrow []$
    \FOR{step $t = 1$ to $T$}
        \STATE $a_{pred} \leftarrow \arg\max \pi_\theta(h)$
        \STATE $\text{Cycle.append}(h)$
        \STATE $\text{Actions.append}(a_{pred})$
        \STATE $h \leftarrow f_\theta(h, o_t)$
    \ENDFOR
    
    \STATE \textit{// Check 1: Is it a closed loop? (Periodicity)}
    \STATE $\delta_{closure} \leftarrow \|h - h_{start}\|_2$
    \IF{$\delta_{closure} > \epsilon$}
        \STATE \textbf{continue} 
    \ENDIF
    
    \STATE \textit{// Check 2: Action Consistency}
    \IF{$\text{Actions} == \mathcal{A}$}
        \STATE $\mathbf{C}_{valid}.\text{add}(\text{Cycle})$
    \ENDIF
\ENDFOR

\STATE {\bfseries Return} $\mathbf{C}_{valid}$
\end{algorithmic}
\end{algorithm}

\begin{figure*}[t]
    \centering
    \includegraphics[width=\textwidth]{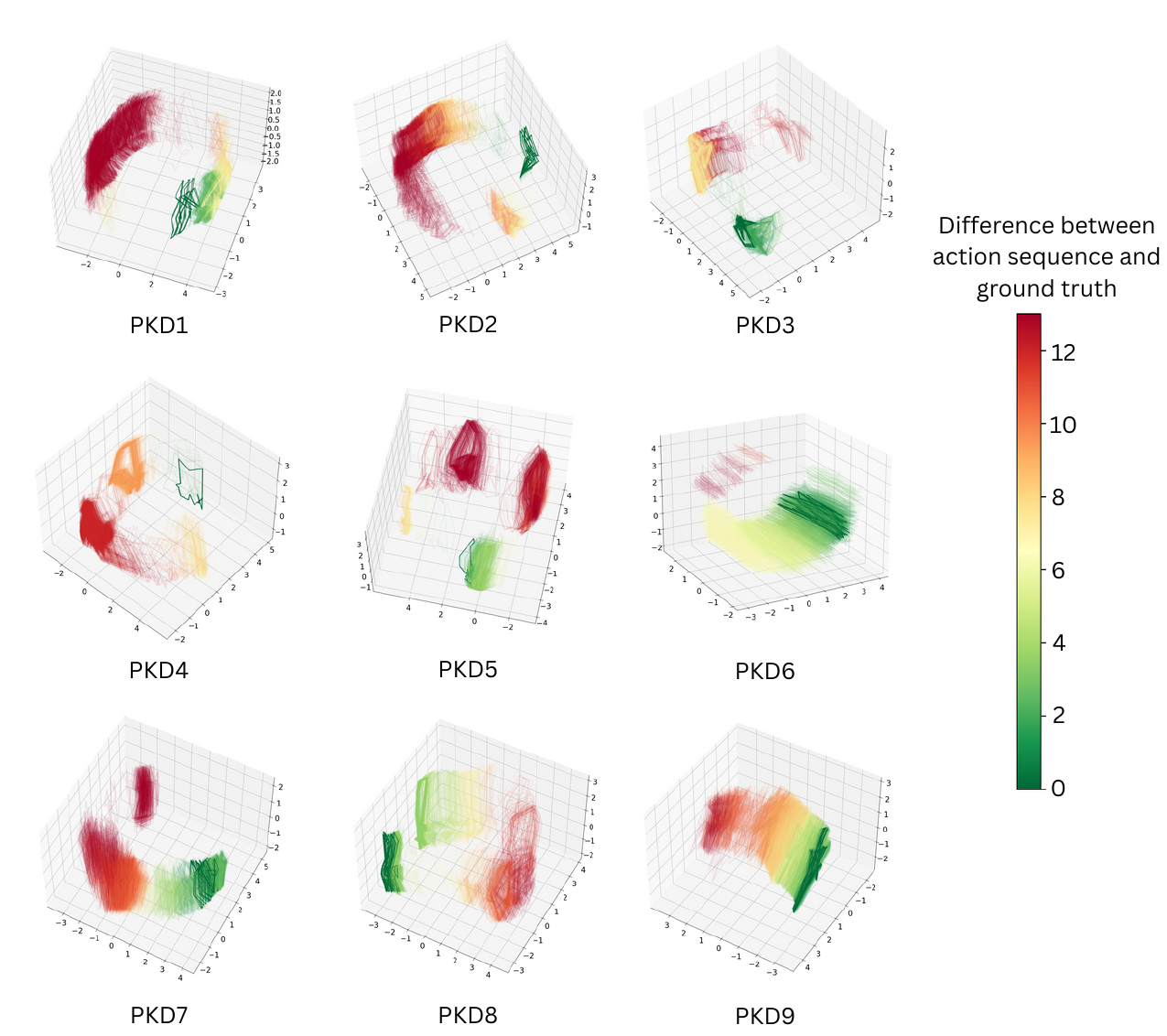}
    \caption{\textbf{Visualizing the landscape of entrained attractors under Periodically-Kicked Drive.} 
    We probe the neural state space of a GRU policy by initializing \textbf{100,000} random hidden states and driving them with a fixed, task-derived PKD sequence. The resulting converged orbits are projected via PCA. 
    \textbf{Color Coding:} Orbits are colored by their \textbf{Action Consistency Error}—the discrepancy between the policy's open-loop action outputs and the ground-truth actions required to sustain the cycle (Green: Low Error/Consistent; Red: High Error/Inconsistent).
    \textbf{Insight:} While the periodic drive entrains a vast number of states into stable periodic orbits (the dissipative "ghost orbits" shown in red), only specific, tightly clustered regions (green) satisfy the causal self-sustainment condition. This visualizes that periodicity provided by resets/PKD is necessary but not sufficient; the learned internal structure must also align to close the sensorimotor loop.}
    \label{fig:acf_clusters}
\end{figure*}

\textbf{Distinguishing Learned Dynamics from Trivial Resets.} The emergence of these self-sustaining limit cycles is \textit{not} a trivial consequence of the episodic task structure. While environmental resets inherently impose periodicity on the observation stream, potentially entraining the neural state into a passive loop, our ACF analysis reveals that this external drive is insufficient for closed-loop stability.

As visualized in Figure~\ref{fig:acf_clusters}, when we probe the state space with 100,000 initializations under a fixed PKD, we find a vast landscape of mathematically stable attractors. However, the majority of these are "ghost orbits" (red) that fail the action-consistency check. They are merely entrained echoes of the input forcing. Valid limit cycles (green) appear only as distinct, isolated clusters where the policy's internal dynamics precisely align with the external drive to produce the correct actions. In a real interaction, any state falling into the inconsistent region would immediately diverge from the cycle after a reset, breaking the loop. Thus, the limit cycles we observe are specific, learned dynamical structures, not artifacts of environmental periodicity.

\clearpage

\section{Behavioral Embedding: Behavioral Potential Field (BPF)}

We introduce the \textbf{Behavioral Potential Field (BPF)} as a method to encode variable-length navigation paths into a fixed-resolution scalar field. This representation allows us to quantify geometric similarities---such as path shape, direction, and detour magnitude---using standard vector space metrics.

\textbf{Geometric Intuition and Validation.} 
Why does the Euclidean distance between two scalar fields represent a meaningful metric for trajectory similarity? Conceptually, BPF transforms a sparse sequence of coordinates into a dense spatial distribution. When we compute the $L_2$ distance between two BPFs, we are calculating the integrated difference in their potential intensities. This approximates a transport-cost metric (similar to the Wasserstein distance) without the computational overhead.

To verify that this metric aligns with human intuition, we analyzed two distinct scenarios using synthetic trajectories, as shown in Figure~\ref{fig:ridge_representation_distances}:

\begin{itemize}
    \item \textbf{Spatial Separation (Figure~\ref{fig:ridge_representation_distances}, Top):} We compared three non-overlapping paths. Visually, \texttt{trj1} is spatially closer to \texttt{trj2} than to the distant \texttt{trj3}. The BPF Euclidean distance captures this relationship accurately: $d(\text{trj1}, \text{trj2}) = 4.38$, which is significantly smaller than $d(\text{trj1}, \text{trj3}) = 7.35$.
    \item \textbf{Shape Deformation (Figure~\ref{fig:ridge_representation_distances}, Bottom):} For paths that overlap (sharing start/end points), the metric distinguishes between minor deviations and major detours. A slight deviation (\texttt{trj2}) yields a distance of $d=2.32$ from the reference \texttt{trj1}, while a larger detour (\texttt{trj3}) yields $d=4.67$.
\end{itemize}

These results confirm that the BPF embedding preserves both the topological shape and spatial displacement of behaviors in a linear metric space, providing a robust basis for CCA alignment.

\begin{figure}[htbp]
    \centering
    \includegraphics[width=\linewidth]{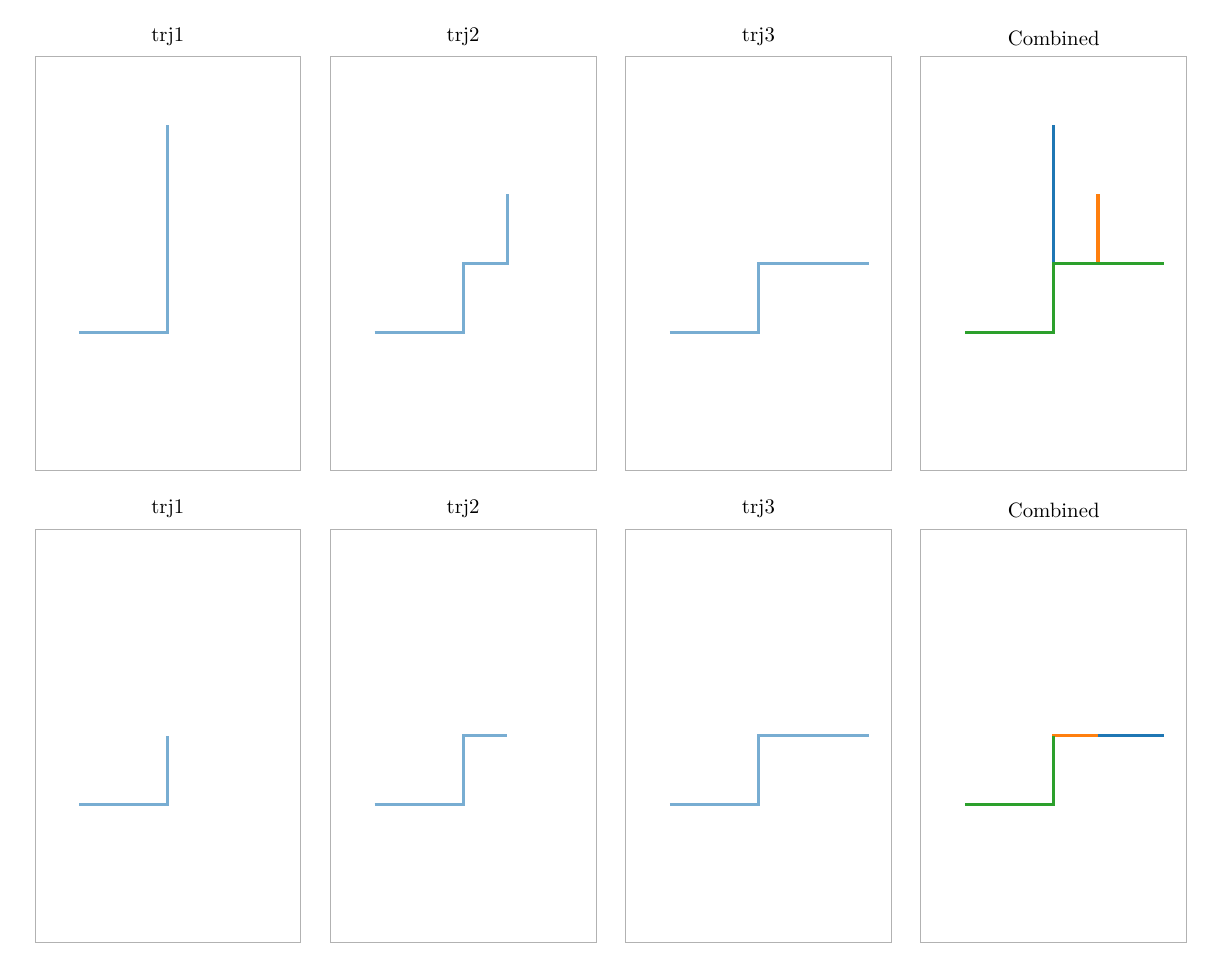}
    \caption{\textbf{Validation of the BPF metric against geometric intuition.}
    \textbf{1. Non-overlapping scenario (Top):} The computed Euclidean distances in BPF space align with the visual spatial separation: $d(\text{trj1}, \text{trj2}) = 4.38 < d(\text{trj1}, \text{trj3}) = 7.35$.
    \textbf{2. Overlapping scenario (Bottom):} For trajectories with varying deformations, the metric reflects the magnitude of the detour: the distance between the reference path and a minor deviation ($d(\text{trj1}, \text{trj2}) = 2.32$) is strictly less than that of a major detour ($d(\text{trj1}, \text{trj3}) = 4.67$).}
    \label{fig:ridge_representation_distances}
\end{figure}

\textbf{Construction Algorithm.}
The formal construction process is detailed in Algorithm~\ref{alg:bpf}. We define a fixed grid (e.g., $21 \times 21$) representing the maze area. For every point $p_k$ in the trajectory, we generate a ``radiance field'' where intensity decays linearly with Euclidean distance. To aggregate the full trajectory, we apply a \textbf{maximum-value fusion} rule. This ensures that the resulting field captures the continuous ``ridge'' of the path without saturation issues common in summation methods.

\begin{algorithm}[tb]
\caption{Construction of Behavioral Potential Field (BPF)}
\label{alg:bpf}
\begin{algorithmic}[1]
\STATE {\bfseries Input:} Trajectory sequence $\tau = \{p_1, p_2, \dots, p_L\}$, Grid dimensions $H \times W$
\STATE {\bfseries Hyperparameters:} Top intensity $\alpha$, Base intensity $\beta$, Effective Radius $R_{\text{eff}}$
\STATE {\bfseries Output:} Behavioral Field Matrix $\mathbf{F}$

\STATE Initialize global field $\mathbf{F} \leftarrow \mathbf{0}_{H \times W}$
\STATE Generate coordinate grids $\mathbf{X}, \mathbf{Y}$ for the $H \times W$ space

\FOR{each point $p_k = (x_k, y_k)$ in $\tau$}
    \STATE \textit{// 1. Compute Euclidean distance from current point to all pixels}
    \STATE $\mathbf{D} \leftarrow \sqrt{(\mathbf{X} - x_k)^2 + (\mathbf{Y} - y_k)^2}$
    
    \STATE \textit{// 2. Generate linear radiance field for this point}
    \STATE $\mathbf{I}_{\text{local}} \leftarrow (\alpha - \beta) \cdot \frac{R_{\text{eff}} - \mathbf{D}}{R_{\text{eff}}} + \beta$
    
    \STATE \textit{// 3. Apply truncation (intensity is 0 beyond effective radius)}
    \STATE $\mathbf{I}_{\text{local}} \leftarrow \text{where}(\mathbf{D} < R_{\text{eff}}, \mathbf{I}_{\text{local}}, 0)$
    
    \STATE \textit{// 4. Max-Fusion: Update global field}
    \STATE $\mathbf{F} \leftarrow \max(\mathbf{F}, \mathbf{I}_{\text{local}})$
\ENDFOR

\STATE {\bfseries Return} $\mathbf{F}$
\end{algorithmic}
\end{algorithm}

\textbf{Hyperparameter Stability.}
While the BPF construction involves specific choices for the effective radius ($R_{\text{eff}}$) and decay profile, the resulting geometric embedding is remarkably robust. The preservation of topological structure does not rely on precise parameter tuning. We provide a detailed empirical verification of this stability in Appendix~\ref{app:bpf_robustness}, where we demonstrate that high canonical correlations persist across a broad sweep of hyperparameter configurations.

\clearpage

\section{Relating Neural and Behavioral Geometries via CCA}\label{app:neu_beh_cca}

To systematically examine the structural correspondence between the agent's internal memory and its physical behavior, we employ \textbf{Canonical Correlation Analysis (CCA)}. Unlike standard correlation which operates on scalar variables, CCA identifies linear relationships between two high-dimensional datasets.

\textbf{Optimization Objective.}
Let $\mathbf{X} \in \mathbb{R}^{N \times d_x}$ denote the dataset of neural hidden states (e.g., GRU vectors) and $\mathbf{Y} \in \mathbb{R}^{N \times d_y}$ denote the corresponding Behavioral Potential Fields (BPFs), where $N$ is the number of samples. CCA seeks pairs of projection vectors $u \in \mathbb{R}^{d_x}$ and $v \in \mathbb{R}^{d_y}$ such that the correlation between the projected variables (canonical variates) is maximized:

\begin{equation}
\label{eq:cca_obj}
\rho = \max_{u, v} \frac{u^T \Sigma_{XY} v}{\sqrt{u^T \Sigma_{XX} u} \sqrt{v^T \Sigma_{YY} v}}
\end{equation}

where $\Sigma_{XX}$ and $\Sigma_{YY}$ are the covariance matrices of neural and behavioral data respectively, and $\Sigma_{XY}$ is the cross-covariance matrix. This optimization is solved iteratively to find orthogonal pairs of directions $(u_i, v_i)$ corresponding to descending correlation values $\rho_i$.

\textbf{Analysis Workflow.}
The complete pipeline is formalized in Algorithm~\ref{alg:cca_pipeline}. We first apply Principal Component Analysis (PCA) to both the neural states and the BPFs. This preprocessing step serves two purposes: it reduces dimensionality to prevent overfitting and acts as a whitening transform. The CCA algorithm then operates on these compressed representations, yielding three key outputs:
\begin{itemize}
    \item \textbf{Canonical Correlations ($\rho$):} A spectrum of scalar values quantifying the strength of alignment for each mode. High values in the first $k$ modes indicate a shared geometric structure of dimensionality $k$.
    \item \textbf{Canonical Vectors ($U, V$):} The projection weights that define \textit{which} features (e.g., specific neurons or spatial regions in BPF) contribute most to the alignment.
    \item \textbf{Canonical Variates:} The coordinates of the data in the shared canonical space, allowing us to visualize and compare the manifolds directly (as seen in Section 5).
\end{itemize}

\begin{algorithm}[tb]
\caption{Neural-Behavioral Alignment Pipeline}
\label{alg:cca_pipeline}
\begin{algorithmic}[1] 
    \STATE {\bfseries Input:} Neural dataset $\mathbf{H} \in \mathbb{R}^{N \times D_{neuron}}$, Behavioral dataset $\mathbf{F} \in \mathbb{R}^{N \times D_{pixel}}$
    \STATE {\bfseries Hyperparameters:} PCA components $k_x, k_y$, CCA components $k_{cca}$

    \STATE 
    \STATE \textit{// Stage 1: Preprocessing \& Basis Rotation}
    \STATE $\mathbf{H}_{pca} \leftarrow \text{PCA}(\mathbf{H}, \text{n\_components}=k_x)$
    \STATE $\mathbf{F}_{pca} \leftarrow \text{PCA}(\mathbf{F}, \text{n\_components}=k_y)$

    \STATE \textit{// Standardize to zero mean and unit variance}
    \STATE $\tilde{\mathbf{H}} \leftarrow \text{StandardScaler}(\mathbf{H}_{pca})$
    \STATE $\tilde{\mathbf{F}} \leftarrow \text{StandardScaler}(\mathbf{F}_{pca})$

    \STATE
    \STATE \textit{// Stage 2: Canonical Correlation Analysis}
    \STATE Solve CCA: find $U, V$ maximizing Eq.~\ref{eq:cca_obj} on $(\tilde{\mathbf{H}}, \tilde{\mathbf{F}})$

    \STATE
    \STATE \textit{// Stage 3: Project to Canonical Space}
    \STATE $\mathbf{Z}_{neural} \leftarrow \tilde{\mathbf{H}} U$ \hfill \textit{\{Neural Canonical Variates\}}
    \STATE $\mathbf{Z}_{behavior} \leftarrow \tilde{\mathbf{F}} V$ \hfill \textit{\{Behavioral Canonical Variates\}}
    \STATE $\rho \leftarrow \text{diag}(\text{corr}(\mathbf{Z}_{neural}, \mathbf{Z}_{behavior}))$

    \STATE {\bfseries Return} Canonical Correlations $\rho$, Variates $\mathbf{Z}_{neural}, \mathbf{Z}_{behavior}$
\end{algorithmic}
\end{algorithm}

\clearpage

\section{CCA Aligned Structures}
\label{app:cca_structures}

In this section, we provide expanded visualizations of the Canonical Correlation Analysis (CCA) results to validate the structural isomorphism between the agent's internal memory and its physical behavior. To ensure the statistical robustness of these geometric portraits, we conducted a large-scale sampling procedure: for each experimental setting (e.g., Jumper-PPO-Mamba), we harvested \textbf{20,000} distinct task instances and their corresponding neural limit cycles.

We project both the neural hidden states and the behavioral potential fields (BPFs) of these \textbf{20,000} samples into their shared canonical space. Figures~\ref{fig:cca_angle_app} and~\ref{fig:cca_distance_app} display the results across different experimental configurations. In each panel, we visualize the geometry spanned by the first 9 Canonical Modes (CM1--CM9). The \textbf{top row} of each subplot represents the manifold structure in the \textbf{neural state space}, while the \textbf{bottom row} represents the corresponding structure in the \textbf{behavioral space}.

\begin{figure*}[htbp]
    \centering
    \includegraphics[width=0.85\textwidth]{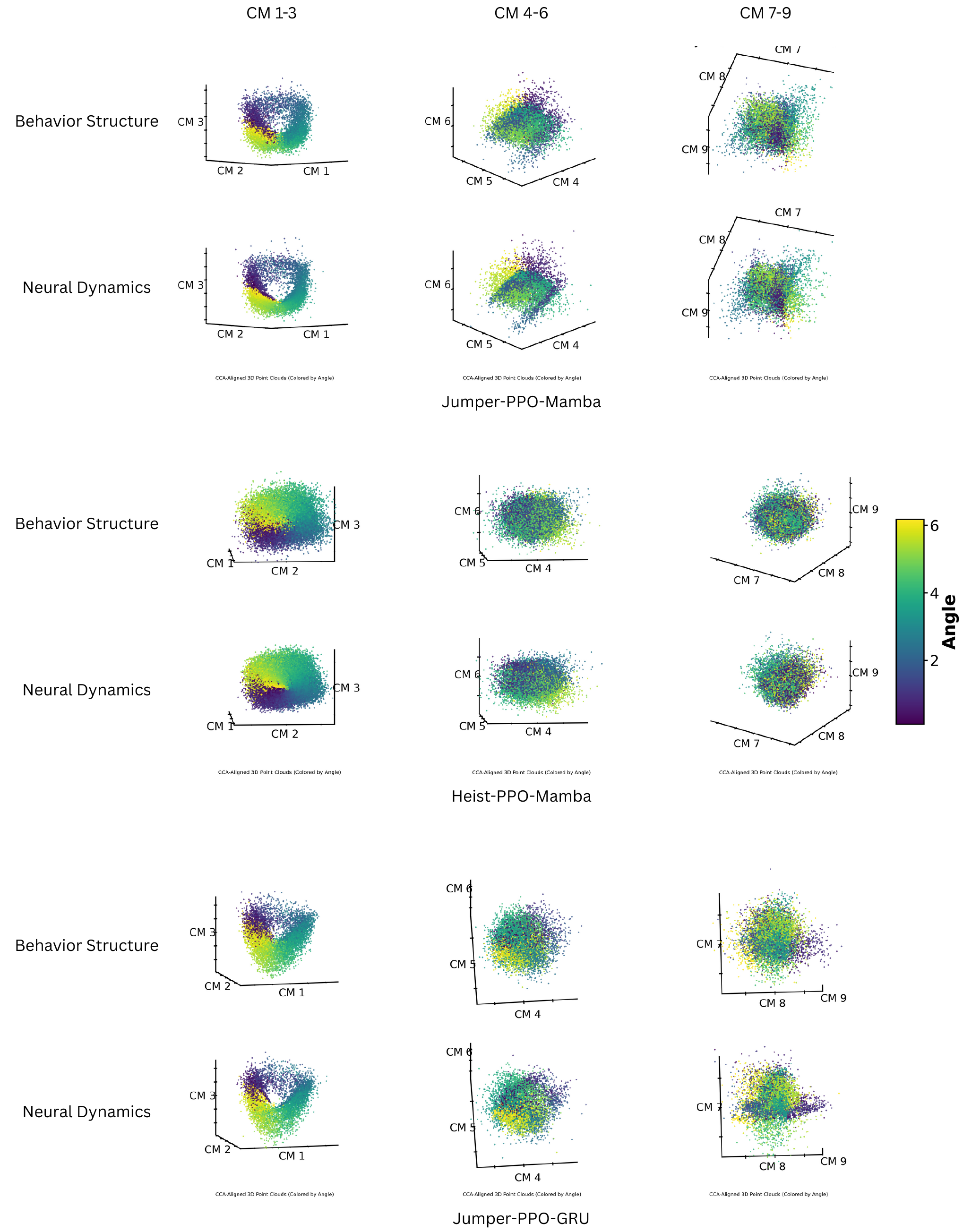}
    \caption{\textbf{CCA-aligned manifolds colored by Cyclic Phase (Angle).} 
    This figure visualizes the topological alignment between neural and behavioral manifolds for a population of \textbf{20,000} sampled limit cycles per task setting. 
    We show the projections onto the first 9 Canonical Modes. 
    The data points are colored by their \textbf{angular phase} along the limit cycle. 
    The similarity in color gradients between the neural projections (top rows) and behavioral projections (bottom rows) across multiple dimensions (CM1--CM9) indicates that the recurrent policy maintains a high-dimensional representation that topologically mirrors the cyclic phase of the physical execution, consistent across the large-scale dataset.}
    \label{fig:cca_angle_app}
\end{figure*}

\begin{figure*}[htbp]
    \centering
    \includegraphics[width=0.85\textwidth]{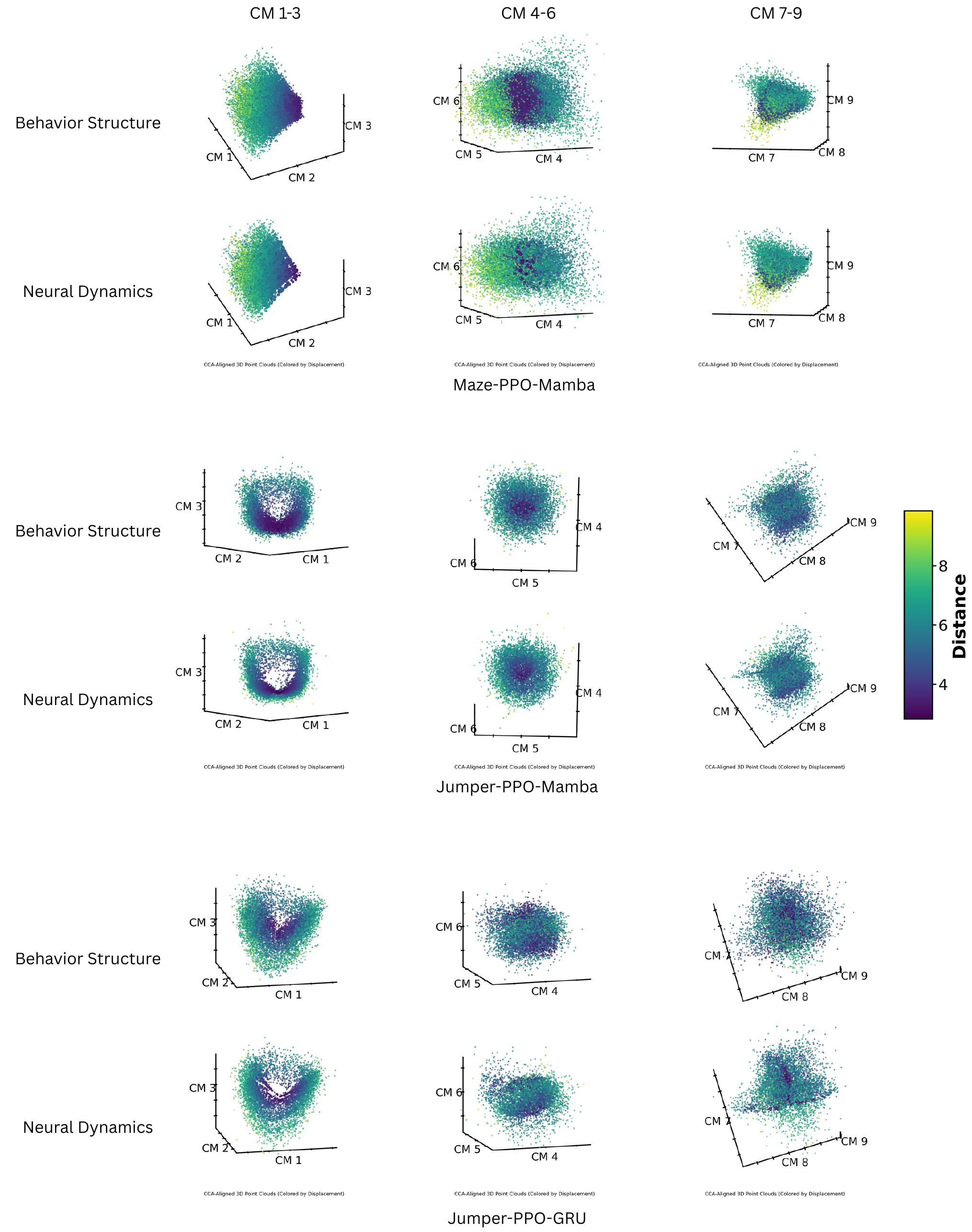}
    \caption{\textbf{CCA-aligned manifolds colored by Trajectory Displacement.} 
    This figure visualizes the metric alignment between neural and behavioral manifolds, again utilizing the \textbf{20,000} samples per setting.
    The same manifolds as in Figure~\ref{fig:cca_angle_app} are now colored by the \textbf{physical length (displacement)} of the trajectory (e.g., from purple/short to yellow/long). 
    The consistent radial gradients observed in both the neural (top rows) and behavioral (bottom rows) canonical variates demonstrate that the neural space preserves an isometric encoding of the physical distance. 
    This suggests the agent's memory acts as a coordinate system that encodes not just "where I am in the cycle" (phase), but "which cycle I am in" (magnitude/geometry).}
    \label{fig:cca_distance_app}
\end{figure*}

\subsection{Distance Matrix Correlation Analysis}

To further verify that the high canonical correlations reflect a genuine preservation of global geometric structure (rather than local linear artifacts), we performed a Representational Similarity Analysis (RSA) comparing the relational geometries of the two spaces.

For each experimental setting, we computed the pairwise Euclidean distance matrices for both the behavioral embeddings ($D_{behav}$) and their corresponding neural limit cycle centroids ($D_{neural}$). We then flattened these distance matrices into 1D vectors and computed the Pearson correlation coefficient ($R$) between them. This metric quantifies the degree of isometry: a high correlation implies that pairs of behaviors that are distinct in the physical space are mapped to proportionally distinct locations in the neural manifold.

As shown in Figure~\ref{fig:dm_pearson}, the neural and behavioral distance matrices exhibit high similarity across different architectures and tasks:
\begin{itemize}
    \item \textbf{Maze-PPO-Mamba:} Pearson $R = 0.9542$, indicating a highly isometric mapping where neural distances tightly track physical transport costs.
    \item \textbf{Jumper-PPO-Mamba:} Pearson $R = 0.8469$.
    \item \textbf{Jumper-PPO-GRU:} Pearson $R = 0.7473$.
\end{itemize}
These results provide quantitative evidence that the neural state space functions as a structured geometric map of the agent's behavioral repertoire.

\begin{figure*}[htbp]
    \centering
    \includegraphics[width=0.7\textwidth]{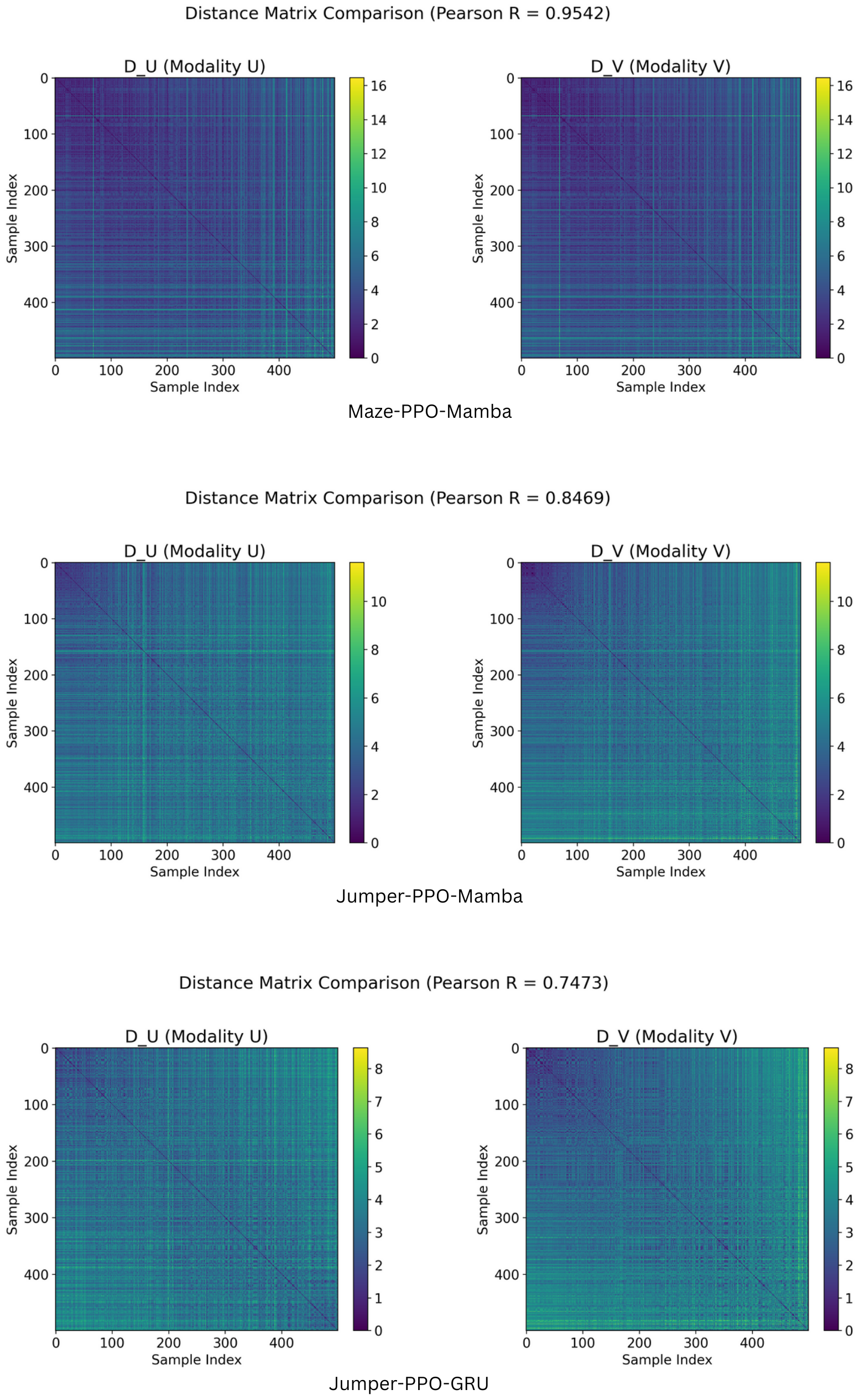}
    \caption{\textbf{Quantitative verification of structural isomorphism via Distance Matrix Comparison.}
    For three distinct experimental settings (Maze-PPO-Mamba, Jumper-PPO-Mamba, Jumper-PPO-GRU), we compare the pairwise Distance Matrices of the Neural modality ($D_U$) and Behavioral modality ($D_V$).
    The color intensity in the matrices represents the pairwise distance between sampled instances. The visual similarity between $D_U$ and $D_V$ is quantified by the Pearson correlation coefficient ($R$).
    The high correlation values ($R > 0.74$ in all cases, reaching $0.95$ for Maze tasks) confirm that the relational geometry—the relative distances between different tasks—is preserved when mapping from the behavioral space to the neural manifold.}
    \label{fig:dm_pearson}
\end{figure*}

\clearpage

\section{Control Experiment: Ruling out Dimensionality Artifacts}
\label{app:control_experiment}

A critical concern in high-dimensional geometric analysis is the potential for ``overfitting'' or statistical artifacts. Specifically, one might hypothesize that the high canonical correlations ($>0.7$ for 10+ dimensions) observed in our main results are merely a consequence of the high dimensionality of the state space ($D_{hidden}=256 \sim 512$), or simply due to the linear constraint of outputting the correct actions, rather than reflecting a learned dynamical structure.

To rigorously rule out these possibilities, we designed a \textbf{Random Control Experiment} on a large scale. We constructed a baseline using an untrained RNN with random weights. Crucially, to ensure a fair comparison, we did not simply sample random noise; we generated a dataset of \textbf{10,000} \textit{Action-Constrained Pseudo-Manifolds}. As detailed in Algorithm~\ref{alg:random_control}, for every state in a real behavioral trajectory, we generated a random neural state and projected it (via Gram-Schmidt orthogonalization) to ensure that it triggers the \textit{exact same action} through the readout layer as the optimized agent.

This procedure creates a ``Pseudo-Neural Limit Cycle'' that is functionally identical to the trained agent (it produces the same behavior) but structurally devoid of learned recurrence. We then performed the same CCA pipeline between these \textbf{10,000} pseudo-cycles and their corresponding behavioral BPFs.

\begin{algorithm}[tb]
\caption{Generation of Action-Constrained Pseudo-Manifolds}
\label{alg:random_control}
\begin{algorithmic}[1]
\STATE {\bfseries Input:} Target behavioral action sequence $A = \{a_1, \dots, a_T\}$, Readout Matrix $W_{out}$, Hidden Dimension $D$
\STATE {\bfseries Output:} Pseudo-Neural Trajectory $H_{pseudo}$

\STATE Initialize $H_{pseudo} \leftarrow []$
\FOR{$t = 1$ to $T$}
    \STATE \textit{// 1. Generate isotropic random noise}
    \STATE $v_{rand} \sim \mathcal{N}(0, I_D)$
    
    \STATE \textit{// 2. Identify target action vector and non-target vectors}
    \STATE $w_{target} \leftarrow W_{out}[a_t]$
    \STATE $W_{other} \leftarrow W_{out}[\setminus a_t]$
    
    \STATE \textit{// 3. Apply Gram-Schmidt / Projection to enforce action constraint}
    \STATE \textit{// Project noise into a subspace where $w_{target}$ dominates $W_{other}$}
    \STATE $v_{constrained} \leftarrow \text{OrthogonalizeAndScale}(v_{rand}, w_{target}, W_{other})$
    
    \STATE \textit{// Verify constraint: $\arg\max(W_{out} v_{constrained}) == a_t$}
    \STATE $H_{pseudo}.\text{append}(v_{constrained})$
\ENDFOR

\STATE {\bfseries Return} $H_{pseudo}$
\end{algorithmic}
\end{algorithm}

\textbf{Results and Discussion.}
The comparison, based on \textbf{10,000} samples for both the control and experimental groups, yields a clear difference that supports the Isometry Hypothesis.

\begin{itemize}
    \item \textbf{Spectral Collapse vs. Sustained Alignment:} As shown in the spectral plots of Figure~\ref{fig:control_comparison_random} and Figure~\ref{fig:control_comparison_gru}, the canonical correlations for the \textbf{10,000} random control samples collapse rapidly after the first 3 dimensions (likely corresponding to the trivial degrees of freedom required to code X/Y/Action). In contrast, the trained GRU sustains high correlations ($>0.7$) for over 10 dimensions across the same sample size.
    
    \item \textbf{Visual Chaos vs. Geometric Order:} The visualizations of the canonical modes (CM) further illustrate this difference. In the trained network (Figure~\ref{fig:control_comparison_gru}), the projections of the neural state (top row) and behavior (bottom row) exhibit identical, smooth geometric gradients across CM0--CM5. In the random control (Figure~\ref{fig:control_comparison_random}), however, the neural projections manifest as unstructured point clouds with no visual correspondence to the behavioral manifold, despite the actions being mathematically forced to match.
\end{itemize}

This large-scale evidence indicates that the high-dimensional alignment we observe is not a trivial artifact of action mimicry or dimensionality, but a specific, emergent property of the \textit{learned recurrent dynamics}.

\begin{figure*}[htbp]
    \centering
    \includegraphics[width=0.7\textwidth]{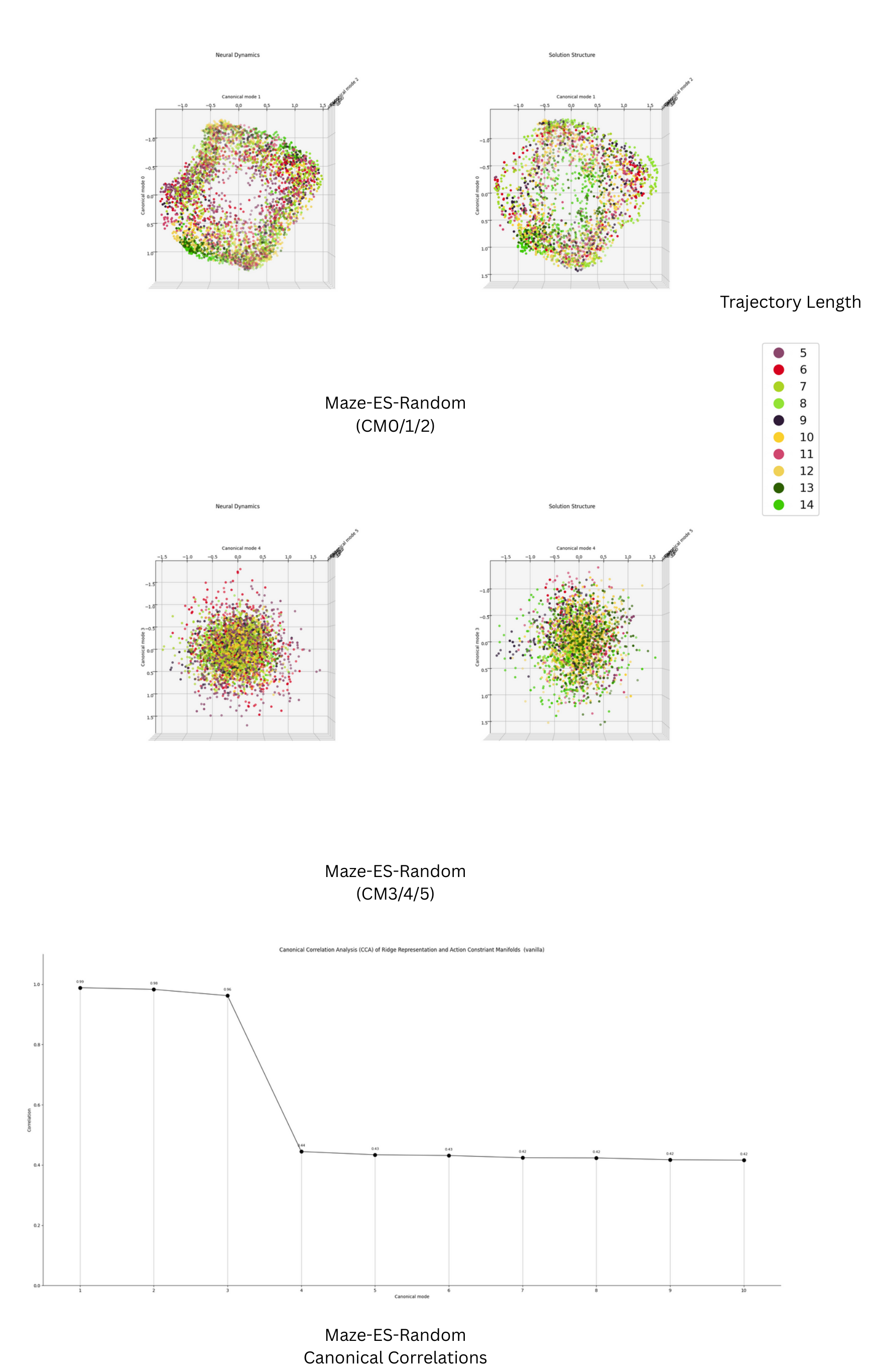}
    \caption{\textbf{Control Result: Action-Constrained Random Network.} 
    Analysis of \textbf{10,000} ``Pseudo-Neural Limit Cycles'' generated from a random RNN constrained to output correct actions. 
    \textbf{(Bottom Plot)} The canonical correlation spectrum drops precipitously after the 3rd dimension, indicating minimal structural depth beyond basic action coding. 
    \textbf{(Scatter Plots)} Projections of the first 6 Canonical Modes (CM0--CM5). Even though the network outputs the correct actions, the neural geometry (visualized with a random distinct color map) appears as a disordered distribution with no topological similarity to the behavioral structure. This confirms that action-matching alone is insufficient to produce the high-dimensional isometry observed in trained agents.}
    \label{fig:control_comparison_random}
\end{figure*}

\begin{figure*}[htbp]
    \centering
    \includegraphics[width=\textwidth]{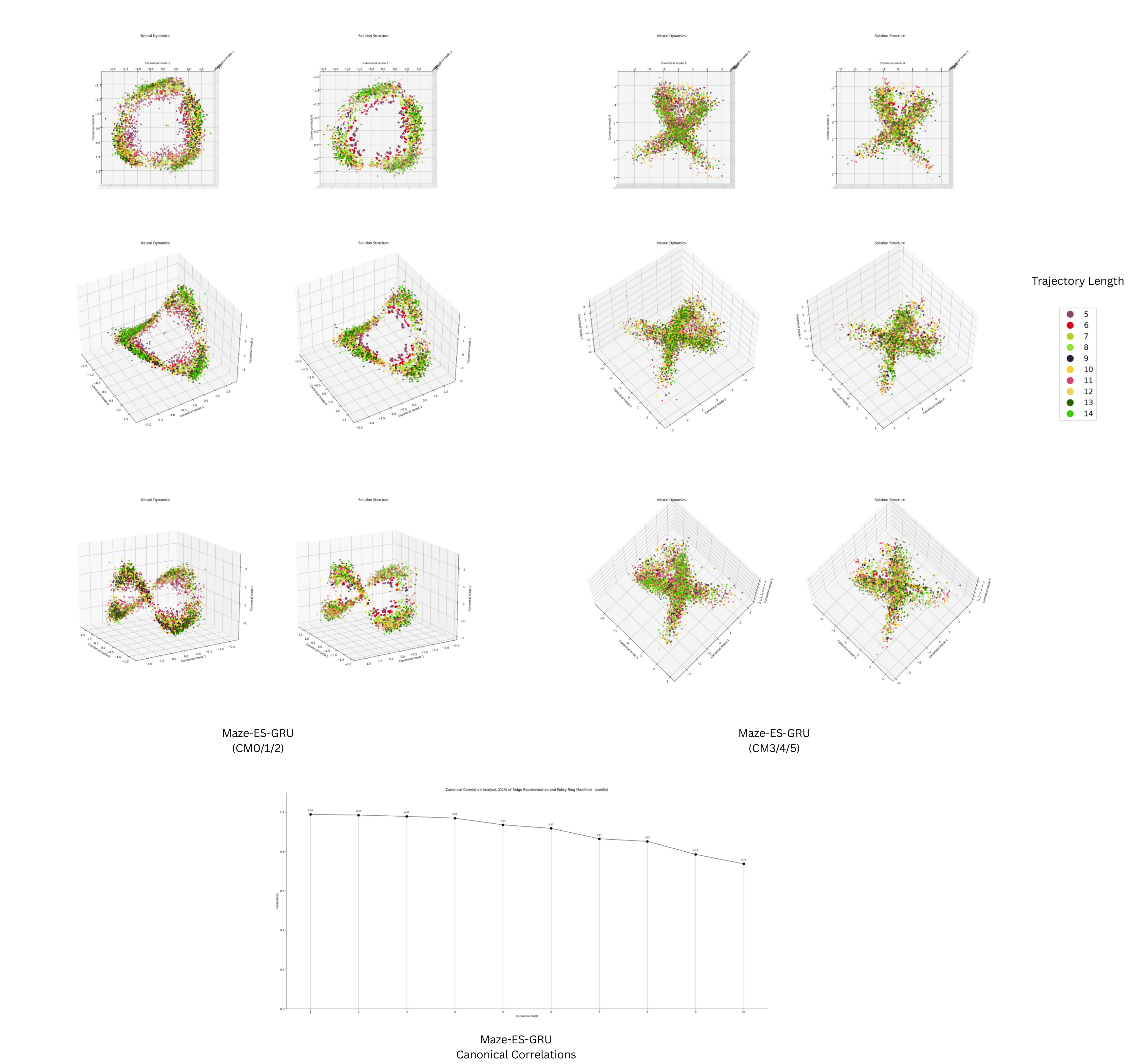}
    \caption{\textbf{Experimental Result: Fully Trained Maze-ES-GRU Agent.} 
    Analysis of \textbf{10,000} actual neural limit cycles sampled from the optimized agent.
    \textbf{(Bottom Plot)} High canonical correlations persist for more than 10 dimensions, suggesting a rich, high-dimensional shared geometry.
    \textbf{(Scatter Plots)} The neural projections (CM0--CM5) exhibit clear, coherent geometric patterns that mirror the behavioral manifold. Unlike the random control, the learned dynamics organize the state space into a structured coordinate system.}
    \label{fig:control_comparison_gru}
\end{figure*}

\clearpage

\section{Robustness of Behavioral Potential Field (BPF) Hyperparameters}
\label{app:bpf_robustness}

A critical question regarding the high canonical correlations reported in our main results is whether they are artifacts of specific hyperparameter choices in the Behavioral Potential Field (BPF) construction. Specifically, one might verify if the alignment relies on a carefully tuned \textit{effective radius} ($R_{\text{eff}}$) or a specific \textit{distance metric} to match the neural geometry.

To rigorously address this concern, we conducted a comprehensive robustness analysis using the Maze-ES-GRU setting. We systematically varied the BPF construction parameters across a wide grid of configurations:
\begin{itemize}
    \item \textbf{Effective Radius Scaling:} We tested four distinct scales for the effective radius relative to the baseline used in the main text: $0.5\times$ (sharp, local fields), $0.75\times$, $1.0\times$ (baseline), and $1.5\times$ (broad, diffuse fields).
    \item \textbf{Distance Metrics:} We evaluated three different norms for calculating the potential decay: $L_1$ (Manhattan distance), $L_2$ (Euclidean distance, baseline), and $L_3$ (Minkowski distance with $p=3$).
\end{itemize}

The results of this grid search are visualized in Figure~\ref{fig:bpf_robustness}. Across all 12 experimental configurations, the canonical correlation spectra exhibit consistent stability. The leading canonical correlations consistently remain high ($>0.85$) and decay slowly across the first 10 modes, regardless of whether the behavioral field is sharp or diffuse, or whether the decay follows an $L_1$ or $L_2$ profile. This invariance confirms that the observed high-dimensional isometry is robust to substantial variations in the specific engineering details of the behavioral embedding we use.

\begin{figure}[htbp]
    \centering
    \includegraphics[width=\textwidth]{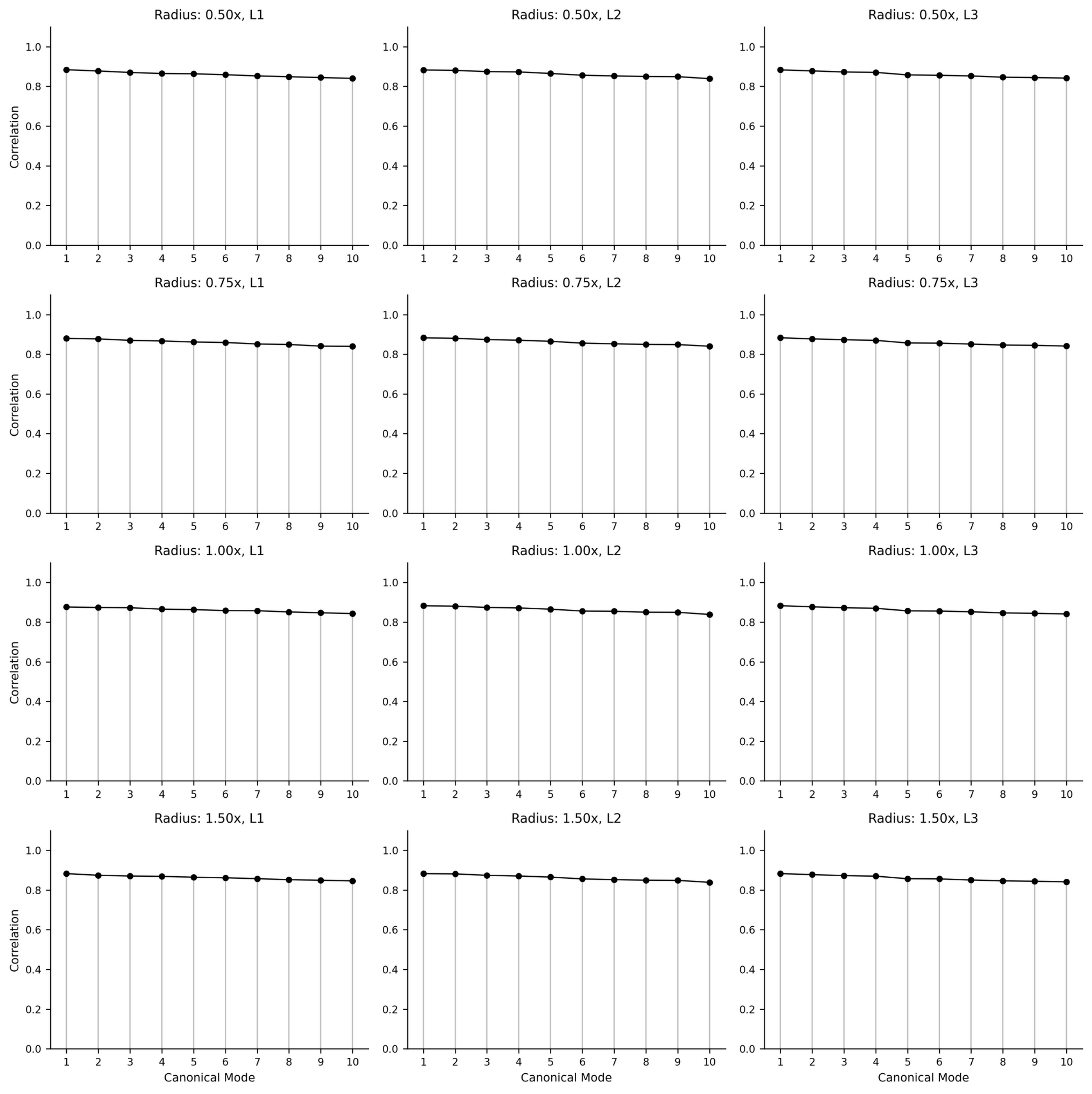}
    \caption{\textbf{Robustness of CCA alignment under BPF hyperparameter variations.}
    We performed a grid search over BPF construction parameters to verify the stability of the neural-behavioral alignment.
    The matrix displays the canonical correlation spectra (top 10 modes) for 12 distinct configurations.
    \textbf{Rows} vary the \textit{Effective Radius} ($R_{\text{eff}}$) from $0.5\times$ to $1.5\times$ the baseline.
    \textbf{Columns} vary the distance metric ($L_1$, $L_2$, $L_3$) used to generate the potential field.
    The consistently high correlations across all panels demonstrate that the structural alignment is robust and not an artifact of parameter fine-tuning.}
    \label{fig:bpf_robustness}
\end{figure}

\clearpage

\section{Counterfactual Verification of CCA-Aligned Representations}
\label{app:counterfactual}

Although Canonical Correlation Analysis (CCA) reveals a high-dimensional alignment between neural dynamics and behavioral geometry, this correlational evidence does not establish functional causality. The aligned dimensions could possibly be epiphenomenal, with the actual control signals residing in the uncorrelated subspaces. To evaluate the causal necessity and sufficiency of the discovered representations, we conducted a counterfactual injection experiment in the Maze-ES-GRU setting.

\subsection{Experimental Design}

Our strategy involves surgically manipulating the agent's internal memory state at the beginning of an episode and observing the impact on adaptation speed.

\textbf{Logic.} In a standard Meta-RL setting, an agent typically requires a "search phase" (exploration) of approximately 250 steps to locate the goal and stabilize its trajectory. However, if we initialize the agent with the hidden state $h^*$ extracted from a converged limit cycle (an "optimal strategy" state), the agent should theoretically bypass exploration and execute the optimal path immediately.

By transforming this optimal state $h^*$ into the CCA canonical space, we can selectively lesion specific dimensions before inverse-projecting them back into the neural space and injecting them into the agent. This allows us to test two competing hypotheses:
\begin{itemize}
    \item \textbf{Causal Necessity:} If the top-ranked CCA dimensions encode the essential navigation map, preserving them (while randomizing all others) should preserve the accelerated convergence.
    \item \textbf{Causal Insufficiency:} Conversely, if we selectively randomize \textit{only} the top CCA dimensions (while keeping all others intact), the agent should lose its "implanted memory" and revert to the slow baseline exploration.
\end{itemize}

\subsection{Intervention Procedure}

The procedure is formalized in Algorithm~\ref{alg:counterfactual}. We utilize the canonical transformation matrices $A$ derived from our main analysis. For a given converged neural state $h^*$, we project it to the latent code $z$. We then apply a masking or randomization operator $\mathcal{M}$ to specific dimensions of $z$—either targeting the top-$K$ correlated modes or the bottom uncorrelated modes. The modified latent code $\tilde{z}$ is mapped back to the RNN's native hidden space via the pseudo-inverse $A^\dagger$ to obtain the counterfactual state $\tilde{h}_{init}$.

\begin{algorithm}[h]
\caption{Counterfactual Injection Test (CCA-Based)}
\label{alg:counterfactual}
\begin{algorithmic}[1]
\STATE {\bfseries Input:} Pre-trained Policy $\pi_\theta$, CCA Projection Matrix $A$, Inverse Projection $A^\dagger$
\STATE {\bfseries Input:} Converged Limit Cycle State $h^*$, Cutoff Dimension $K$
\STATE {\bfseries Parameter:} Condition Type $\in \{\textsc{Keep-Top}, \textsc{Remove-Top}\}$

\STATE \textit{// 1. Project optimal state to Canonical Space}
\STATE $z_{cca} \leftarrow \text{Project}(h^*, A)$

\STATE \textit{// 2. Apply Causal Intervention}
\IF{Condition is \textsc{Keep-Top}}
    \STATE $\tilde{z}_{cca} \leftarrow \text{Concat}(z_{cca}[0:K], \text{RandomNoise}(dim-K))$
    \STATE \textit{\{Preserve high-correlation semantics, destroy residuals\}}
\ELSE
    \STATE $\tilde{z}_{cca} \leftarrow \text{Concat}(\text{RandomNoise}(K), z_{cca}[K:])$
    \STATE \textit{\{Destroy high-correlation semantics, preserve residuals\}}
\ENDIF

\STATE \textit{// 3. Inverse Project to Neural Space}
\STATE $\tilde{h}_{init} \leftarrow \text{InverseProject}(\tilde{z}_{cca}, A^\dagger)$

\STATE \textit{// 4. Run Evaluation Episode}
\STATE Initialize Environment $E$
\STATE Set Agent Memory $h_0 \leftarrow \tilde{h}_{init}$
\STATE Run episode and record \textbf{Convergence Time} $T_{conv}$

\STATE {\bfseries Return} $T_{conv}$
\end{algorithmic}
\end{algorithm}

\subsection{Results}

Figure~\ref{fig:counterfactual_histograms} presents the histograms of convergence times under different intervention conditions.

\begin{itemize}
    \item \textbf{Baseline (Fig.~\ref{fig:counterfactual_histograms}a):} In standard episodes (cold start), the distribution of convergence times peaks around \textbf{250 steps}, representing the typical exploration cost.
    \item \textbf{Oracle Injection (Fig.~\ref{fig:counterfactual_histograms}b):} Directly injecting the unaltered optimal state $h^*$ shifts the peak dramatically to \textbf{50 steps}, confirming that $h^*$ encodes the solution.
    \item \textbf{Causal Validation (Fig.~\ref{fig:counterfactual_histograms}c):} Crucially, when we retain only the top CCA dimensions and \textbf{randomize all remaining dimensions}, the convergence peak remains at \textbf{50 steps}. This proves that the navigation-critical information is almost entirely compressed within the high-CCA subspace; the uncorrelated dimensions are causally irrelevant for this task.
    \item \textbf{Negative Control (Fig.~\ref{fig:counterfactual_histograms}d):} Conversely, when we randomize the top CCA dimensions while preserving the rest of the network state intact, the performance collapses back to the baseline (peak $\approx$ 250 steps). The agent effectively suffers "amnesia," confirming that the high-CCA dimensions are not just correlated, but are the \textbf{key drivers} of the learned behavior.
\end{itemize}

\begin{figure*}[t]
    \centering
    \includegraphics[width=\textwidth]{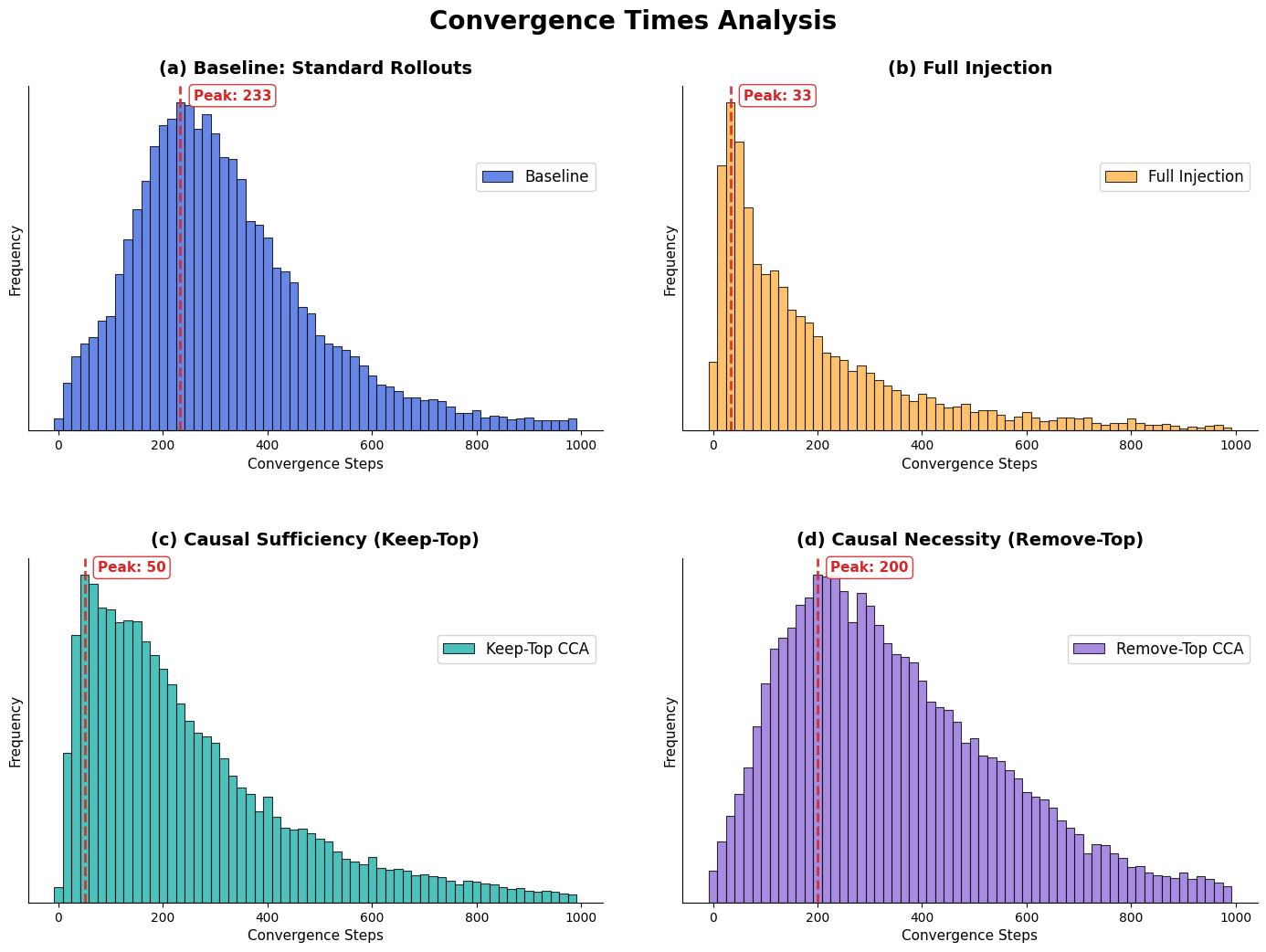}
    \caption{\textbf{Histograms of convergence times under counterfactual neural state injection.}
    \textbf{(a)} \textbf{Baseline:} Standard rollouts start with zero memory; agents typically require $\sim$250 steps (exploration phase) to stabilize on the optimal path.
    \textbf{(b)} \textbf{Full Injection:} Injecting the exact hidden state from a converged limit cycle allows the agent to execute the optimal strategy immediately (peak $\sim$50 steps).
    \textbf{(c)} \textbf{Causal Sufficiency (Keep-Top):} Even if all neural dimensions \textit{except} the top CCA modes are randomized, the agent retains its immediate performance (peak $\sim$50 steps). This underscores that the CCA-aligned manifold captures the sufficient statistics for control.
    \textbf{(d)} \textbf{Causal Necessity (Remove-Top):} Randomizing only the top CCA dimensions—while keeping all other neural components intact—reverts the agent to baseline performance (peak $\sim$250 steps), confirming that the injected high-correlation dimensions are the necessary drivers of navigation success.}
    \label{fig:counterfactual_histograms}
\end{figure*}

\end{document}